\documentclass{article}
\usepackage{arxiv}
\usepackage{amsmath,amsthm}
\usepackage[utf8]{inputenc}
\usepackage[T1]{fontenc}
\usepackage{hyperref}
\usepackage{url}
\usepackage{booktabs}
\usepackage{amsfonts}
\usepackage{nicefrac}
\usepackage{array}
\usepackage{booktabs}
\usepackage{fancyvrb}
\usepackage{graphicx}
\usepackage{caption,subcaption}
\usepackage{placeins}
\usepackage{algorithm}
\usepackage{algpseudocode}
\usepackage{textcomp}
\usepackage{rotating}
\usepackage{array}
\usepackage{booktabs}
\usepackage{tabularx}
\usepackage{fancyvrb}
\usepackage{cuted}
\usepackage{capt-of}
\usepackage{longtable}
\usepackage{hyperref} 

\DeclareUnicodeCharacter{2212}{\textminus}

\newtheorem{proposition}{Proposition}

\newcommand{\piold}{\pi_{\theta_{\text{old}}}}
\newcommand{\pinew}{\pi_{\theta}}
\newcommand{\piref}{\pi_{\text{ref}}}
\newcommand{\aOR}{\alpha_{\text{OR}}}   % OR trust-region radius
\newcommand{\aL}{\alpha_{\text{LoRA}}}  % LoRA scaling --- never bare \alpha
\newcommand{\DKL}{D_{\mathrm{KL}}}
\newcommand{\sg}{\operatorname{sg}}
\title{OR Else: A Differentiable Trust Region for Policy Optimization}

\author{
  Chinmay Rane \\
  Quantiphi Inc \\
  Marlborough, MA, 01752 \\
  \texttt{ranechinmay54@gmail.com}
  \And
  Kanishka Tyagi \\
  Self Machines Inc \\
  San Francisco, California, 94402 \\
  \texttt{kanishkaugee@icloud.com}
  \And
  Michael Manry \\
  Department of Electrical Engineering \\
  The University of Texas at Arlington \\
  Arlington, TX, 76010 \\
  \texttt{manry@uta.edu}
}

\begin{document}
\raggedbottom
\setlength{\textfloatsep}{12pt plus 2pt minus 2pt}
\maketitle

% ============================================================
% ============================================================
\begin{abstract}
PPO and the GRPO baseline studied here use clipped surrogate objectives
whose favorable-direction saturation introduces an abrupt change in the
scalar objective's derivative. We ask whether Output Reset (OR), a
smooth one-sided saturation rule, offers a useful alternative for large
language model post-training. PPO-OR and GRPO-OR replace the clipped
policy term with an OR squared-margin loss in rollout-relative token
log-ratio space; the advantage sign determines the update direction, and
a token contributes zero direct OR residual after crossing the favorable
margin. We compare PPO-clip with PPO-OR under generalized advantage estimation
(GAE), and GRPO with GRPO-OR under group-relative advantages, using
\texttt{Llama-3.2-1B-Instruct} on Anthropic \texttt{hh-rlhf} with one
shared reward model and three seeds per method. Under GAE, PPO-OR has a
mean final training-time reward-model score $0.305$ higher than PPO-clip,
with a larger observed across-seed spread. Under group-relative
advantages, GRPO-OR does not have a higher mean score, but shows a
smaller observed spread, a near-zero terminal OR residual, and a
declining overshoot fraction, while the matched GRPO clipped-objective
trace remains variable. Both group-relative methods exhibit
substantially larger rollout-to-current log-ratio displacement than the
GAE methods, and OR does not consistently reduce it. Thus, OR changes optimization behavior in both matched comparisons, but
the observed reward effect differs between them. At $G=2$, the
GRPO-OR diagnostics do not translate into a reward-score gain. Whether
larger groups change this outcome remains open. The reported scores are
training-time reward-model measurements, not held-out human-preference
performance.
\end{abstract}

\noindent\textbf{Code:}
\url{https://github.com/ChicoR-Dr/PPO-GRPO-OR}


\newpage

% ============================================================
\section{Introduction}
\label{sec:intro}

Reinforcement learning from human feedback (RLHF) is widely used to
adapt large language models to human preferences, building on
preference-based reinforcement learning and subsequent work on
human-feedback training for language models
~\cite{CHRISTIANO2017,STIENNON2020,OUYANG2022,BAIETAL2022}. Trust Region
Policy Optimization (TRPO) introduced an explicit KL-constrained policy
update, while Proximal Policy Optimization (PPO) replaced that
second-order procedure with a clipped first-order surrogate
~\cite{SCHULMAN2015,SCHULMAN2017}. PPO remains a common optimization
method for LLM post-training and uses the clipped surrogate objective
\begin{equation}
  L^{\text{CLIP}}(\theta)
  =
  \hat{\mathbb{E}}_t\left[
    \min\left(
      r_t(\theta)\hat{A}_t,\;
      \operatorname{clip}\left(
        r_t(\theta),\,1-\epsilon,\,1+\epsilon
      \right)\hat{A}_t
    \right)
  \right].
  \label{eq:ppo_clip}
\end{equation}
Here,
$r_t(\theta)=\pinew(a_t\mid s_t)/\piold(a_t\mid s_t)$ is the
probability ratio and $\hat{A}_t$ is an estimated advantage. For
$\hat{A}_t>0$, the clipped branch removes further incentive to increase
$r_t$ once $r_t>1+\epsilon$; for $\hat{A}_t<0$, it removes further
incentive to decrease $r_t$ once $r_t<1-\epsilon$. The objective remains
sensitive to movement in the opposite, disadvantageous direction.
Clipping therefore suppresses further policy-loss updates after a
sample has moved sufficiently far in the advantage-favored direction,
but it does not impose a hard bound on the optimizer step, parameter
displacement, KL divergence, or cumulative policy drift. The
one-sided zero-gradient region is an intentional feature of PPO, not by
itself an optimization pathology. The distinction examined in this work
is the geometry of the transition into that region: although
$L^{\text{CLIP}}$ is continuous, its derivative with respect to $r_t$
changes abruptly at the active clipping boundary. Group Relative Policy
Optimization (GRPO)~\cite{SHAO2024} replaces the learned value baseline
with advantages computed from groups of responses to the same prompt,
but retains the clipped probability-ratio surrogate and its piecewise
gradient structure.

We introduce PPO-OR and GRPO-OR, which replace the clipped policy term
with an objective derived from Output Reset (OR)~\cite{TYAGI2024}.
Conceptually, both methods retain the advantage-estimation machinery of
their parent algorithms but change how each sampled token is acted upon:
the sign of the advantage specifies whether the token should become more
or less probable, and the direct OR update for that token stops once its
log-probability change crosses a prescribed margin in the desired
direction. PPO-OR uses GAE advantages and a learned value head, whereas
GRPO-OR uses group-relative advantages without a critic.

Adapting OR to policy optimization requires a scalar quantity that
represents how the probability of a sampled token changes relative to
the rollout policy. We use the per-token log-probability ratio
\begin{equation}
\rho_t
=
\log\pinew(a_t\mid s_t)
-
\log\piold(a_t\mid s_t),
\end{equation}
which plays the role of the network output in the original OR
formulation. For a nonzero advantage, OR assigns a target of $+\aOR$
when $\hat{A}_t>0$ and $-\aOR$ when $\hat{A}_t<0$. Once $\rho_t$
crosses the corresponding boundary in the advantage-favored direction,
the target is reset to the current log-ratio, producing zero residual
and zero direct gradient from that sample. In forward-value form, the
per-token objective is equivalent to the squared one-sided margin
\begin{equation}
  \ell_{\mathrm{OR},t}
  =
  \begin{cases}
    \left[
      \max\left(
        0,\,
        \aOR-\operatorname{sign}(\hat{A}_t)\rho_t
      \right)
    \right]^2,
    & \hat{A}_t \neq 0, \\[4pt]
    0,
    & \hat{A}_t = 0.
  \end{cases}
  \label{eq:or_intro}
\end{equation}

The OR-MSE gradient approaches zero continuously at the favorable
saturation boundary and remains zero beyond it, although it is not
globally bounded on the adverse side of the active branch. PPO-OR is
therefore not merely a differentiable smoothing of PPO clipping. It
simultaneously replaces the probability ratio $r_t$ with the log-ratio
$\rho_t$, replaces the clipped surrogate with a squared one-sided
margin, and uses the sign rather than the magnitude of $\hat{A}_t$ to
determine the update direction. On the active branch, gradient descent
on the OR-MSE term increases the selected token's log-probability when
$\hat{A}_t>0$ and decreases it when $\hat{A}_t<0$.
Proposition~\ref{prop:sign_consistency} formalizes this sample-level
directional property, while Proposition~\ref{prop:diff} establishes
continuity of the first derivative at the favorable boundary. These
properties do not imply that the aggregate update is identical to the
full policy gradient or that optimizing the OR-MSE objective guarantees
monotonic improvement in expected return.

We evaluate the loss replacement in two advantage-estimation settings.
PPO-clip and PPO-OR use generalized advantage estimation with a learned
value head, whereas GRPO and the proposed GRPO-OR use group-relative
advantages with $G=2$. The study therefore contains two matched
within-family substitutions: PPO-clip versus PPO-OR under GAE, and GRPO
versus GRPO-OR under group-relative advantages. Cross-family comparisons
also reflect differences in advantage estimation, value-model usage,
response length, and response sampling, and are not interpreted as a
controlled factorial decomposition. All four methods are evaluated on
Anthropic \texttt{hh-rlhf}~\cite{BAIETAL2022} using a
Llama-3.2-1B-Instruct policy, a single shared frozen reward model, and
three seeds. Because the experiments use a lightweight reward model, a
small group size, and a resource-constrained training configuration,
the results characterize the evaluated setting rather than establish
general superiority over PPO or GRPO.

Under GAE, the observed mean endpoint reward-model score increases from
$+0.195$ for PPO-clip to $+0.500$ for PPO-OR, an absolute difference of
$+0.305$; the corresponding across-seed standard deviations are
$0.110$ and $0.158$. Under $G=2$ group-relative advantages, the
observed mean changes from $+0.488$ for GRPO to $+0.457$ for GRPO-OR,
while the across-seed standard deviation decreases from $0.080$ to
$0.027$. Because the reward model is trained from pairwise differences
and its numerical zero is not calibrated, these scores are interpreted
comparatively rather than as positive or negative success thresholds.
The recorded policy-objective trajectories also differ substantially:
the clipped GRPO objective exhibits large excursions, whereas the
GRPO-OR MSE remains close to zero on its own scale. These raw values are
not directly comparable because the objectives have different
definitions, signs, and numerical scales, so we treat them only as
objective-specific training diagnostics. Policy drift provides a shared
diagnostic across all four methods. Measured by the batch-averaged
absolute log-ratio relative to the rollout policy, both group-relative
methods exhibit substantially greater drift than the GAE-based methods,
and GRPO-OR does not consistently reduce this drift relative to GRPO.
This finding separates local gradient saturation from global policy
control: OR suppresses the direct policy-loss contribution of a sample
after it crosses its favorable boundary, but does not constrain
cumulative movement caused by shared parameters, auxiliary losses,
later minibatches, or changing advantage signs.

Our contributions are:
\begin{enumerate}
\item We introduce \textbf{PPO-OR} and \textbf{GRPO-OR}, two LLM
post-training variants that replace the clipped policy
surrogate with an Output-Reset MSE objective in log-ratio
space while retaining, respectively, GAE and group-relative
advantage estimation.

\item We characterize the OR policy objective as a continuously
differentiable, one-sided squared-margin loss. We establish
that its sample-level update changes a selected token's
log-probability in the direction indicated by the sign of its
estimated advantage and vanishes continuously at the favorable
saturation boundary.

\item We conduct two matched, three-seed comparisons of clipping and
OR: PPO-clip versus PPO-OR under GAE, and GRPO versus GRPO-OR
under $G=2$ group-relative advantages. In the evaluated setting,
PPO-OR attains a higher observed mean endpoint reward-model
score than PPO-clip, whereas GRPO-OR attains a slightly lower
observed mean than GRPO but a smaller across-seed standard
deviation.

\item We show that replacing clipping with OR does not resolve the
large policy drift observed under $G=2$ group-relative
advantages. This negative result identifies cumulative policy
drift, rather than boundary smoothness alone, as the principal
unresolved problem exposed by the study
(Sec.~\ref{sec:dynamics}).
\end{enumerate}

% ============================================================
\section{Background and Related Work}
\label{sec:background}

\subsection{Policy Optimization in RLHF}
\label{sec:background_ppo}

Reinforcement learning from human feedback (RLHF) adapts a language-model
policy using a scalar reward learned from human preference comparisons.
A common large-scale pipeline consists of supervised fine-tuning,
reward-model training on preference pairs, and online policy
optimization against the learned reward
~\cite{CHRISTIANO2017,STIENNON2020,OUYANG2022,BAIETAL2022}. The
policy-optimization stage is commonly expressed as a KL-regularized
objective,
\begin{equation}
  \begin{aligned}
    J(\theta)
    ={}&
    \mathbb{E}_{\substack{x\sim\mathcal{D}\\
                          y\sim\pinew(\cdot\mid x)}}
    \left[r_\phi(x,y)\right] \\
    &-\beta\,
    \mathbb{E}_{x\sim\mathcal{D}}
    \left[
      \DKL\left(
        \pinew(\cdot\mid x)
        \,\middle\|\,
        \piref(\cdot\mid x)
      \right)
    \right],
  \end{aligned}
  \label{eq:rlhf_obj}
\end{equation}
where $r_\phi(x,y)$ is the reward-model score and $\beta$ controls
regularization toward the frozen reference policy $\piref$. For a response
$y_i=(y_{i,1},\ldots,y_{i,T_i})$ sampled from the rollout policy $\piold$, the
sampled rollout-to-reference log-ratio at token $t$ is
\begin{equation}
  \begin{aligned}
    k^{\mathrm{roll}}_{i,t}
    ={}&
    \log\piold\left(
      y_{i,t}\mid x_i,y_{i,<t}
    \right) \\
    &-
    \log\piref\left(
      y_{i,t}\mid x_i,y_{i,<t}
    \right).
  \end{aligned}
  \label{eq:sampled_kl}
\end{equation}
A standard terminal-reward construction assigns the reward-model score to the
last generated token and applies the sampled reference penalty at each response
position,
\begin{equation}
  r'_{i,t}
  =
  r_\phi(x_i,y_i)\,\mathbf{1}[t=T_i]
  -\beta k^{\mathrm{roll}}_{i,t}.
  \label{eq:shaped_reward}
\end{equation}
The exact placement of this reference term is an implementation choice; the
GAE-based methods in this work use the shaped-reward form in
Eq.~\eqref{eq:shaped_reward}, whereas the group-relative methods apply the
sampled reference term directly in the optimization objective as described in
Sec.~\ref{sec:ppo_or_loss}.

At a more fundamental level, policy-gradient methods optimize expected
return from sampled trajectories. REINFORCE uses the score-function
estimator to weight action log-probability gradients by sampled returns
~\cite{WILLIAMS1992}, while the policy-gradient theorem provides the
corresponding expected-gradient foundation for parameterized policies
~\cite{SUTTON1999}. Advantage baselines and learned value functions are
then used to reduce estimator variance without changing the intended
policy-gradient direction.

Trust Region Policy Optimization (TRPO)~\cite{SCHULMAN2015} controls policy
movement through an explicit KL-constrained update, implemented using an
approximate second-order procedure. Proximal Policy Optimization
(PPO)~\cite{SCHULMAN2017} replaces that constrained update with a clipped
surrogate that supports multiple minibatch epochs using first-order
optimization. In RLHF, PPO is typically paired with a learned value function
and Generalized Advantage Estimation (GAE)~\cite{SCHULMAN2016} to obtain
token-level advantages $\hat{A}_{i,t}$. The frozen reference policy $\piref$
and rollout policy $\piold$ serve different purposes: $\piref$ anchors the
model to the initial instruction-tuned policy, while $\piold$ generates the
current batch and defines the local importance ratio
$r_t(\theta)=\pinew(a_t\mid s_t)/\piold(a_t\mid s_t)$ in the policy objective.

Empirical analyses have shown that observed PPO and TRPO behavior depends
substantially on implementation details and on how the surrogate,
advantage estimates, and value function behave in practice
~\cite{ENGSTROM2020,ILYAS2020}. Clipping also does not by itself impose
the formal KL-constrained trust region used by TRPO
~\cite{WANG2020}. In addition, PPO likelihood ratios and advantage
estimates can induce heavy-tailed policy-gradient behavior
~\cite{GARG2021}. More recent work has analyzed stationary-point
convergence of PPO-Clip under explicit assumptions~\cite{JIN2024} and
developed alternative first-order proximal objectives with stronger
control of probability-ratio movement~\cite{XIE2025}. These findings
motivate studying the geometry of the policy objective without treating
clipping as either a complete trust-region mechanism or an inherently
defective design.

Several alignment methods simplify or avoid this PPO configuration. Direct
Preference Optimization (DPO)~\cite{RAFAILOV2023} trains directly on offline
preference pairs without fitting an explicit reward model or sampling from the
policy during fine-tuning. ReMax~\cite{LI2024} and
RLOO~\cite{AHMADIAN2024} instead retain online REINFORCE-style optimization
while replacing PPO's learned critic with trajectory-level variance-reduction
baselines. These methods alter the data regime, baseline construction, or
policy-optimization procedure. PPO-OR remains an online policy-optimization
method and specifically replaces the PPO-style clipped policy term with a
$C^1$ squared one-sided margin in token-level log-ratio space; it does not
introduce a formal trust-region constraint.

\subsection{Group-Relative Advantage Estimation}
\label{sec:background_grpo}

Group Relative Policy Optimization (GRPO) was introduced in
DeepSeekMath~\cite{SHAO2024}. It removes the learned value function and
estimates response quality from multiple completions sampled for the
same prompt. Given $G$ responses with reward-model scores
$\{r_j\}_{j=1}^{G}$, the group-relative advantage for response $i$ is
\begin{equation}
  \hat{A}^{\mathrm{GR}}_i
  =
  \frac{
    r_i-\operatorname{mean}\left(\{r_j\}_{j=1}^{G}\right)
  }{
    \max\left(
      \operatorname{std}\left(\{r_j\}_{j=1}^{G}\right),
      \delta
    \right)
  },
  \label{eq:grpo_adv}
\end{equation}
where $\delta>0$ prevents division by a zero or near-zero group standard
deviation. A positive advantage indicates that a response scores above the
group mean, while a negative advantage indicates that it scores below the
mean. The response-level value is assigned to all valid generated tokens of
that response. This construction removes critic training and storage, but the
informativeness and variability of the estimate depend on group size, reward
diversity, and the convention used to compute the group standard deviation.
DeepSeek-R1 later applied group-relative reinforcement learning at larger
scale for reasoning post-training~\cite{DEEPSEEK2025}.

Standard GRPO combines the group-relative estimator with a PPO-style clipped
probability-ratio objective, so the advantage estimator and policy loss remain
separable design choices. Subsequent methods modify other parts of this
pipeline. DAPO~\cite{BYTEDANCE2025} uses asymmetric clipping together with
dynamic sampling, token-level loss aggregation, and overlength reward shaping
to address entropy collapse, ineffective groups, and long-response training
instability. Dr.~GRPO~\cite{DRGRPO} identifies an optimization bias that can
favor increasing response length and proposes a corrected objective that
improves token efficiency. In contrast, GRPO-OR retains group-relative
advantages but replaces the clipped surrogate with the same OR policy loss
used by PPO-OR. The experiments therefore compare the clipped and OR losses
within the group-relative family rather than treating GRPO-OR as a new
advantage estimator.

\subsection{Output Reset}
\label{sec:background_or}

Output Reset (OR)~\cite{TYAGI2024} was introduced as a target-adjustment
procedure for classifier training with mean squared error. For a training
pattern $p$, let $y_p(i_c)$ denote the pre-activation output for the correct
class and $y_p(i_d)$ an output for an incorrect class. OR assigns
\begin{equation}
  \begin{aligned}
    t_p(i_c)
    &=
    \begin{cases}
      1, & y_p(i_c)\leq 1, \\
      y_p(i_c), & y_p(i_c)>1,
    \end{cases} \\
    t_p(i_d)
    &=
    \begin{cases}
      0, & y_p(i_d)\geq 0, \\
      y_p(i_d), & y_p(i_d)<0.
    \end{cases}
  \end{aligned}
  \label{eq:or_class}
\end{equation}
Before an output crosses its desired boundary, its target remains fixed at
$1$ for the correct class or $0$ for an incorrect class. After crossing the
boundary in the desired direction, OR resets the target to the current output,
which makes the corresponding MSE residual and direct gradient contribution
zero during that evaluation.

The present work transfers this target-reset construction from classifier
outputs to token-level policy optimization. The classifier output becomes the
log-probability change relative to the rollout policy, while the correct versus
incorrect class distinction is replaced by the sign of the estimated
advantage. This mapping yields a squared one-sided margin with continuous first
derivative, as developed in Sec.~\ref{sec:method}. Unlike TRPO, it does not
constrain KL divergence or parameter displacement; unlike PPO and GRPO, it
does not use the clipped minimum surrogate; and unlike DPO, it remains an
online reward-driven policy-optimization objective.

% ============================================================

\section{PPO-OR}
\label{sec:method}

PPO-OR replaces the clipped policy-loss term while retaining the remaining components of the corresponding policy-optimization pipeline. For each sampled token, the method measures the change in log-probability relative to the rollout policy and uses the sign of the estimated advantage to determine whether that token should become more or less probable. When this change crosses a prescribed margin in the advantage-favored direction, Output Reset assigns zero residual to the token. We first establish the correspondence between the original Output Reset formulation and token-level policy optimization, then define the PPO-OR and GRPO-OR objectives, and finally characterize the analytical properties of the resulting scalar loss.

\subsection{From Output Reset to Policy Optimization}
\label{sec:mapping}

Adapting Output Reset from classification to policy optimization requires a scalar output and a desired direction for that output. For a sampled token $a_t$ in state $s_t$, we use the per-token log-probability ratio
\begin{equation}
  \rho_t = \log\pinew(a_t\mid s_t) - \log\piold(a_t\mid s_t),
  \label{eq:log_ratio}
\end{equation}
which measures the change in the selected token's log-probability relative to the rollout policy. We further define the advantage-sign variable
\begin{equation}
  \sigma_t = \operatorname{sign}(\hat{A}_t) \in \{-1,0,+1\}.
  \label{eq:advantage_sign}
\end{equation}
A positive estimated advantage gives $\sigma_t=+1$, indicating that the sampled token should become more probable, whereas a negative estimated advantage gives $\sigma_t=-1$, indicating that it should become less probable. When $\hat{A}_t=0$, the OR policy-loss term supplies no directional update for that token. Table~\ref{tab:or_mapping} summarizes the resulting correspondence with the original classification formulation.

\begin{table*}[t]
\centering
\caption{Conceptual correspondence between Output Reset in classification and PPO-OR in token-level policy optimization. The parameter $\aOR>0$ is a one-sided margin in log-ratio space; it is not a hard bound on parameter movement, policy displacement, probability ratios, or KL divergence.}
\label{tab:or_mapping}
\small
\setlength{\tabcolsep}{5pt}
\renewcommand{\arraystretch}{1.12}
\begin{tabular}{@{}p{0.22\textwidth}p{0.28\textwidth}p{0.42\textwidth}@{}}
\toprule
\textbf{OR classification} & \textbf{PPO-OR} & \textbf{Role} \\
\midrule
Pre-activation output $y_p(i)$ & Log-ratio $\rho_t$ & Change in the selected token's log-probability relative to $\piold$ \\
Correct-class output & $\hat{A}_t>0$ & Favor an increase in the selected token's probability \\
Incorrect-class output & $\hat{A}_t<0$ & Favor a decrease in the selected token's probability \\
Fixed class target & Signed target margin $\sigma_t\aOR$ & Desired movement in the advantage-favored direction \\
Output crosses its target & $\sigma_t\rho_t>\aOR$ & Reset the target to $\rho_t$, producing zero residual \\
No class-specific signal & $\hat{A}_t=0$ & Assign zero contribution from the OR policy-loss term \\
\bottomrule
\end{tabular}
\end{table*}

The margin $\aOR$ is analogous to PPO's clipping parameter only in the limited sense that both determine where favorable-direction saturation begins. PPO's ratio-space interval $[1-\epsilon,1+\epsilon]$ corresponds to
\begin{equation}
  \left[\log(1-\epsilon),\,\log(1+\epsilon)\right]
  \label{eq:ppo_log_bounds}
\end{equation}
in log-ratio space. The positive boundary has magnitude $\log(1+\epsilon)$, whereas the negative boundary has magnitude $-\log(1-\epsilon)$; a symmetric OR margin therefore cannot match both PPO clipping boundaries exactly with a single value of $\aOR$. For example, when $\epsilon=0.2$, the PPO boundaries are approximately $[-0.223,+0.182]$, whereas $\aOR=0.2$ gives symmetric OR margins at $[-0.2,+0.2]$.

\subsection{PPO-OR and GRPO-OR Objectives}
\label{sec:ppo_or_loss}

Using $\sigma_t=\operatorname{sign}(\hat{A}_t)$, the OR-adjusted target is
\begin{equation}
  \tau_t =
  \begin{cases}
    \sigma_t\aOR, & \sigma_t\neq 0 \text{ and } \sigma_t\rho_t\leq\aOR, \\
    \rho_t, & \text{otherwise}.
  \end{cases}
  \label{eq:or_target}
\end{equation}
When the signed log-ratio has not crossed the favorable margin, the target remains fixed at $\sigma_t\aOR$; once $\sigma_t\rho_t>\aOR$, the target follows the current log-ratio and the residual becomes zero. The same zero-residual branch is used when $\hat{A}_t=0$. Let $\mathcal{M}$ denote the set of valid generated-token positions in a minibatch, excluding prompt and padding positions, and let $N_{\mathcal{M}}=\lvert\mathcal{M}\rvert$. The OR policy loss is
\begin{equation}
  L_{\mathrm{OR}}(\theta) =
  \frac{1}{N_{\mathcal{M}}}
  \sum_{(b,t)\in\mathcal{M}}
  \left(\rho_{b,t}-\sg[\tau_{b,t}]\right)^2,
  \label{eq:l_or}
\end{equation}
where $\sg[\cdot]$ denotes stop-gradient. Averaging over valid generated tokens prevents prompt and padding positions from affecting the normalization. The forward value of the corresponding per-token loss is equivalently
\begin{equation}
  \ell_{\mathrm{OR},t} =
  \begin{cases}
    \left[\max\left(0,\,\aOR-\sigma_t\rho_t\right)\right]^2, & \sigma_t\neq 0, \\
    0, & \sigma_t=0.
  \end{cases}
  \label{eq:or_margin_loss}
\end{equation}
This form makes the one-sided saturation geometry explicit: the loss and its derivative approach zero continuously at the favorable boundary and remain zero beyond it. The derivative is not globally bounded on the adverse side of the active branch; for example, when $\sigma_t=+1$, the active branch contains every $\rho_t<\aOR$, and the derivative magnitude grows without bound as $\rho_t\rightarrow-\infty$. PPO-OR differs from PPO clipping in more than the smoothness of its saturation boundary. On its active branch, the PPO clipped surrogate weights a sample using the magnitude of its estimated advantage, whereas PPO-OR uses only $\sigma_t$ to determine the target direction; its scalar gradient coefficient is set by the distance between $\rho_t$ and the OR margin rather than by $\lvert\hat{A}_t\rvert$. Because a nonzero advantage enters the loss only through its sign, PPO-OR is invariant to positive rescaling:
\begin{equation}
  \ell_{\mathrm{OR},t}(\rho_t,c\hat{A}_t)
  = \ell_{\mathrm{OR},t}(\rho_t,\hat{A}_t),
  \qquad c>0,\;\hat{A}_t\neq 0.
  \label{eq:or_scaling_invariance}
\end{equation}
This property does not imply robustness to arbitrary advantage noise, because a perturbation that changes the advantage sign reverses the target direction. The sign-only construction is therefore a deliberate design choice, not a claim that advantage magnitudes are generally uninformative or that sign-only weighting is universally preferable. For PPO-OR, the signs $\sigma_{b,t}$ are obtained from token-level GAE advantages and a learned value function. The reference-policy contribution is incorporated into the shaped rewards before GAE is computed. For a generated token, define the rollout-time sampled reference log-ratio
\begin{equation}
  k^{\mathrm{roll}}_{b,t}
  = \log\piold(a_{b,t}\mid s_{b,t})
  - \log\piref(a_{b,t}\mid s_{b,t}),
  \label{eq:rollout_reference_ratio}
\end{equation}
and the shaped reward
\begin{equation}
  r'_{b,t}
  = r_{\mathrm{RM},b}\,\mathbf{1}[t=T_b]
  - \beta k^{\mathrm{roll}}_{b,t},
  \label{eq:ppo_or_shaped_reward}
\end{equation}
where $r_{\mathrm{RM},b}$ is the response-level reward-model score and $T_b$ is the final valid generated position. The resulting training loss is
\begin{equation}
  L_{\mathrm{PPO\text{-}OR}}(\theta,\phi)
  = L_{\mathrm{OR}}(\theta)
  + c_1 L_{\mathrm{value}}(\phi)
  - c_2 H_{\mathcal{M}}[\pinew],
  \label{eq:l_total}
\end{equation}
where $L_{\mathrm{value}}$ is the unclipped critic mean-squared error against the GAE return targets, $H_{\mathcal{M}}[\pinew]$ is the entropy averaged over valid generated positions, and $\phi$ denotes the value-function parameters. No additional direct reference-policy term appears in Eq.~\eqref{eq:l_total}, because the penalty has already contributed to the rewards from which the returns and GAE advantages were constructed.

GRPO-OR uses the same OR policy loss but replaces the token-level GAE advantage with the response-level group-relative advantage from Eq.~\eqref{eq:grpo_adv}. For response $i$, the sign
\begin{equation}
  \sigma_i^{\mathrm{GR}}
  = \operatorname{sign}\left(\hat{A}^{\mathrm{GR}}_i\right)
  \label{eq:grpo_advantage_sign}
\end{equation}
is assigned to every valid generated token belonging to that response; groups with exactly zero reward variance are excluded from the policy update. This changes the source and granularity of the advantage signal while leaving the OR target construction unchanged. In the group-relative implementation, reference-policy regularization is applied directly to the optimization loss through the sampled current-to-reference log-ratio
\begin{equation}
  k^{\mathrm{cur}}_{b,t}(\theta)
  = \log\pinew(a_{b,t}\mid s_{b,t})
  - \log\piref(a_{b,t}\mid s_{b,t})
  \label{eq:current_reference_ratio}
\end{equation}
and its masked mean
\begin{equation}
  L_{\mathrm{ref}}(\theta)
  = \frac{1}{N_{\mathcal{M}}}
  \sum_{(b,t)\in\mathcal{M}} k^{\mathrm{cur}}_{b,t}(\theta).
  \label{eq:reference_penalty}
\end{equation}
This quantity is a sampled-token reference log-ratio penalty rather than the full categorical KL over the vocabulary at every state. The GRPO-OR training loss is
\begin{equation}
  L_{\mathrm{GRPO\text{-}OR}}(\theta)
  = L_{\mathrm{OR}}(\theta)
  + \beta L_{\mathrm{ref}}(\theta)
  - c_2 H_{\mathcal{M}}[\pinew].
  \label{eq:grpo_or_total}
\end{equation}
GRPO-OR contains no value-loss term because it does not train a critic, and the direct reference-policy contribution is applied once rather than also being included in the reward-model scores used to construct the group-relative advantages.

\subsection{Analytical Properties}
\label{sec:theory}

The following propositions characterize the scalar OR policy loss. They do not establish a hard trust region, a monotonic expected-return guarantee, or a bound on the aggregate neural-network update.

\begin{proposition}[One-sided saturation and smoothness]
\label{prop:diff}
For a fixed nonzero advantage sign $\sigma_t$, the per-token loss in Eq.~\eqref{eq:or_margin_loss} is positive when $\sigma_t\rho_t<\aOR$ and zero when $\sigma_t\rho_t\geq\aOR$. Its derivative is
\begin{equation}
  \frac{\partial\ell_{\mathrm{OR},t}}{\partial\rho_t} =
  \begin{cases}
    -2\sigma_t\left(\aOR-\sigma_t\rho_t\right), & \sigma_t\rho_t<\aOR, \\
    0, & \sigma_t\rho_t\geq\aOR.
  \end{cases}
  \label{eq:or_derivative}
\end{equation}
The derivative approaches zero continuously at the favorable boundary and remains zero beyond it. The scalar loss is therefore $C^1$ but not $C^2$ at $\sigma_t\rho_t=\aOR$. On the strict overshoot branch, $\tau_t=\rho_t$, so the corresponding token contributes zero residual and zero direct OR policy-loss gradient during that evaluation.
\end{proposition}

\begin{proof}
For $\sigma_t\neq0$, Eq.~\eqref{eq:or_margin_loss} reduces to $\ell_{\mathrm{OR},t}=[\max(0,\aOR-\sigma_t\rho_t)]^2$. When $\sigma_t\rho_t<\aOR$, differentiating the active squared term gives the first branch of Eq.~\eqref{eq:or_derivative}; when $\sigma_t\rho_t>\aOR$, the loss is identically zero and hence has zero derivative. At $\sigma_t\rho_t=\aOR$, the active-branch derivative also equals zero, so the first derivative is continuous while the second derivative changes from $2$ on the active branch to $0$ on the saturated branch. For strict overshoot, Eq.~\eqref{eq:or_target} assigns $\tau_t=\rho_t$, yielding zero residual and zero direct gradient from that token-level term.
\end{proof}

\begin{proposition}[Sample-level update direction]
\label{prop:sign_consistency}
On the active branch, the negative gradient of an individual PPO-OR policy-loss term is a positive scalar multiple of $\sigma_t\nabla_\theta\log\pinew(a_t\mid s_t)$. Consequently, the isolated steepest-descent direction favors increasing the selected token's log-probability when $\hat{A}_t>0$ and decreasing it when $\hat{A}_t<0$.
\end{proposition}

\begin{proof}
Because the rollout policy is fixed during the optimization epochs, its log-probability contributes no parameter gradient. On the active branch,
\begin{equation}
  -\nabla_\theta\ell_{\mathrm{OR},t}
  = 2\sigma_t\left(\aOR-\sigma_t\rho_t\right)
  \nabla_\theta\log\pinew(a_t\mid s_t).
  \label{eq:or_negative_gradient}
\end{equation}
The scalar factor $2(\aOR-\sigma_t\rho_t)$ is positive. Equivalently, with $d_t=-\nabla_\theta\ell_{\mathrm{OR},t}$, the directional derivative of the selected token's log-probability is
\begin{equation}
  \begin{aligned}
    \nabla_\theta\log\pinew(a_t\mid s_t)^{\top}d_t
    &= 2\sigma_t\left(\aOR-\sigma_t\rho_t\right) \\
    &\quad\times
    \left\|\nabla_\theta\log\pinew(a_t\mid s_t)\right\|_2^2.
  \end{aligned}
\end{equation}
Except at stationary points where the log-probability gradient is zero, its sign is determined by $\sigma_t$, which proves the stated sample-level direction.
\end{proof}

Proposition~\ref{prop:diff} characterizes one-sided saturation of the scalar loss, and Proposition~\ref{prop:sign_consistency} characterizes the isolated steepest-descent direction of one token-level term. Neither describes the exact AdamW update used in the experiments, which also depends on optimizer state, weight decay, other tokens, and the value, entropy, and reference-policy components of the minibatch objective. Because PPO-OR discards $\lvert\hat{A}_t\rvert$ and reweights samples according to their distance from the OR margin, its aggregate minibatch gradient is not generally identical or collinear with the full magnitude-weighted policy gradient. The results above therefore provide neither a monotonic expected-return guarantee nor a bound on parameter displacement, probability-ratio movement, or KL divergence.

%===================================================

\section{Experiments}
\label{sec:experiments}

\subsection{Experimental Setup}
\label{sec:setup}

All experiments were conducted on a single NVIDIA GeForce RTX
4060~Ti GPU with 16.7~GB of reported available VRAM. We used
PyTorch~2.12 with bfloat16 precision for the trainable policy and 4-bit
NormalFloat quantization through BitsAndBytes~\cite{DETTMERS2023} for
the frozen reference model and reward-model backbone. Policy adaptation
used PEFT~\cite{MANGRULKAR2022} and LoRA~\cite{HU2022}. Each method
was evaluated with seeds 42, 43, and 44 on the same machine and with the
same frozen reward model. Sharing one scorer removes reward-model
variation across policy-optimization runs, but the four configurations
do not form a fully controlled factorial comparison because the GAE and
group-relative families differ in advantage estimation, critic usage,
sampling procedure, and response length. We therefore interpret the
study as two matched within-family comparisons: PPO-clip versus PPO-OR
under GAE, and GRPO versus GRPO-OR under group-relative advantages.

The trainable policy was initialized from
\texttt{meta-llama/Llama-3.2-1B-Instruct}; we did not train a new
supervised fine-tuning checkpoint. A frozen copy of the same checkpoint
served as the reference policy $\piref$, while $\piold$ denotes the
rollout-policy snapshot associated with the current rollout batch.
Anthropic \texttt{hh-rlhf}~\cite{BAIETAL2022} was used for both reward-
model training and policy optimization. Preference pairs were used for
the former, whereas the human turn alone was used as the generation
prompt for the latter. Reusing one dataset source for both phases is a
deliberate simplification and may allow overlap between reward-model
training prompts and policy-optimization prompts; this limitation is
discussed in Sec.~\ref{sec:limitations}.

\subsubsection{Reward Model Training}
\label{sec:rm_training}

The reward model consists of a frozen, 4-bit quantized
Llama-3.2-1B backbone and a trainable one-hidden-layer MLP head,

\begin{equation}
\operatorname{Linear}(2048,512)
\rightarrow \operatorname{ReLU}
\rightarrow \operatorname{Dropout}(0.1)
\rightarrow \operatorname{Linear}(512,1)
\label{eq:rm_head}
\end{equation}

The head contains approximately one million trainable parameters, and
only these parameters receive gradients. For each prompt--response
sequence, the hidden state at the final non-padding position is passed
through the MLP to produce the scalar score $r_\phi(x,y)$.

For a preferred response $y^+$ and a rejected response $y^-$ to the
same prompt $x$, the reward head is trained with the Bradley--Terry
loss ~\cite{BRADLEY1952}
\begin{equation}
  \mathcal{L}_{\mathrm{BT}}
  = -\log\sigma\left(
      r_\phi(x,y^+) - r_\phi(x,y^-)
    \right).
  \label{eq:bt_loss}
\end{equation}
The loss equals $\log 2\approx0.693$ when the two responses receive
equal scores and approaches zero as the preferred response receives an
increasingly larger score. Because the objective depends only on score
differences, the numerical zero of the learned reward scale is not
identified.

We trained the reward head on 8,000 preference pairs for three epochs
using a batch size of 8, AdamW with learning rate $2\times10^{-4}$, and
cosine annealing. Prompt--response sequences were truncated or padded
to a maximum length of 512 tokens. Pairwise preference accuracy was
computed as
\begin{equation}
  \operatorname{Acc}_{\mathrm{pair}}
  = \frac{1}{N_{\mathrm{eval}}}
    \sum_{i=1}^{N_{\mathrm{eval}}}
    \mathbf{1}\left[
      r_\phi(x_i,y_i^+) > r_\phi(x_i,y_i^-)
    \right].
  \label{eq:rm_pairwise_accuracy}
\end{equation}
The resulting checkpoint achieved 60--61\% pairwise accuracy (in-sample) and was
then frozen and reused for every policy-optimization run.

\begin{algorithm}[t]
\caption{Reward-model training on Anthropic \texttt{hh-rlhf}}
\label{alg:rm_training}
\begin{algorithmic}[1]
\Require Preference set
$\mathcal{D}=\{(x_i,y_i^+,y_i^-)\}_{i=1}^{8000}$,
frozen backbone $f_{\mathrm{LM}}$, reward head $h_\psi$
\For{epoch $e=1,\dots,3$}
  \For{minibatch $\mathcal{B}\subset\mathcal{D}$,
       $|\mathcal{B}|=8$}
    \State Tokenize $(x,y^+)$ and $(x,y^-)$ to maximum length 512
    \State $\mathbf{h}^+\gets
      f_{\mathrm{LM}}(x,y^+)[\text{last non-padding position}]$
    \State $\mathbf{h}^-\gets
      f_{\mathrm{LM}}(x,y^-)[\text{last non-padding position}]$
    \State $r^+\gets h_\psi(\mathbf{h}^+)$ and
           $r^-\gets h_\psi(\mathbf{h}^-)$
    \State $\mathcal{L}_{\mathrm{BT}}\gets
      -|\mathcal{B}|^{-1}
      \sum_{(x,y^+,y^-)\in\mathcal{B}}\log\sigma(r^+-r^-)$
    \State Update $\psi$ using AdamW
  \EndFor
\EndFor
\State \Return frozen reward model
$r_\phi\equiv h_\psi\circ f_{\mathrm{LM}}$
\end{algorithmic}
\end{algorithm}

\subsubsection{Policy Optimization}
\label{sec:rl_training}

All four methods were trained for 500 rollout steps using four prompts
per rollout batch, four optimization epochs per batch, and learning
rate $10^{-5}$. Generation used temperature $0.9$ and top-$p$ $0.9$ ~\cite{HOLTZMAN2020}.
The reference coefficient was $\beta=0.05$, and the entropy coefficient
was $c_2=0.01$. The GAE-based methods additionally used $\gamma=1.0$,
$\lambda=0.95$, and value-loss coefficient $c_1=0.1$. Optimization used
AdamW with a cosine schedule and a warmup period equal to 10\% of the
total rollout steps. Policy weights were updated through LoRA adapters
with rank $r=16$ and scaling parameter $\aL=32$, applied to
\texttt{q\_proj} and \texttt{v\_proj}; this yielded approximately four
million trainable policy parameters out of 1.2 billion total
parameters.

Table~\ref{tab:rl_configurations} summarizes the method-specific
settings. PPO-clip and PPO-OR use a learned value head
$\operatorname{Linear}(2048,1)$ initialized from
$\mathcal{N}(0,0.01)$ and an unclipped mean-squared value objective.
Their sampled reference-policy term is incorporated into the shaped
rewards before GAE, as defined in Sec.~\ref{sec:ppo_or_loss}; no second
direct reference penalty is added to their optimization loss. PPO-clip
uses $\epsilon=0.2$, whereas PPO-OR uses $\aOR=0.2$. These settings use
the same nominal numerical value but not identical boundaries: PPO
clipping maps to
$[\log(0.8),\log(1.2)]\approx[-0.223,+0.182]$ in log-ratio space,
whereas PPO-OR uses symmetric margins at $[-0.2,+0.2]$.

\begin{table*}[t]
\centering
\caption{Policy-optimization configurations. The matched comparisons
replace only the policy-loss term within each advantage family. The
GAE and group-relative families additionally differ in critic usage,
sampling procedure, and maximum response length.}
\label{tab:rl_configurations}
\small
\setlength{\tabcolsep}{3.5pt}
\renewcommand{\arraystretch}{1.12}
\begin{tabular}{@{}p{0.12\textwidth}p{0.23\textwidth}p{0.18\textwidth}p{0.30\textwidth}p{0.09\textwidth}@{}}
\toprule
\textbf{Method} &
\textbf{Advantage and critic} &
\textbf{Policy term} &
\textbf{Reference regularization} &
\textbf{Response cap} \\
\midrule
PPO-clip &
Token-level GAE with learned value head &
Clipped surrogate, Eq.~\eqref{eq:ppo_clip} &
Sampled reference term incorporated into shaped rewards before GAE &
128 tokens \\
PPO-OR &
Token-level GAE with learned value head &
OR loss, Eq.~\eqref{eq:l_or} &
Sampled reference term incorporated into shaped rewards before GAE &
128 tokens \\
GRPO &
Group-relative response advantage with $G=2$; no critic &
Clipped surrogate, Eq.~\eqref{eq:ppo_clip} &
Sampled current-to-reference penalty added directly to the loss &
64 tokens \\
GRPO-OR &
Group-relative response advantage with $G=2$; no critic &
OR loss, Eq.~\eqref{eq:l_or} &
Sampled current-to-reference penalty added directly to the loss &
64 tokens \\
\bottomrule
\end{tabular}
\end{table*}

For GRPO and GRPO-OR, two responses were generated for each prompt.
The implementation computes a group-relative advantage from the two
reward-model scores, clamps the computed group standard deviation below
at $10^{-4}$, and excludes groups with exactly zero reward variance
from the policy update. For every non-tied $G=2$ group, the two
responses therefore receive opposite advantage signs, although their
normalized magnitudes depend on the standard-deviation convention used
by the installed PyTorch version. This distinction affects the
magnitude-weighted GRPO clipped surrogate, whereas GRPO-OR uses only
the signs to construct its targets. The group-relative methods compute
advantages from the reward-model scores and add the sampled
current-to-reference penalty directly to the optimization loss, as
defined in Sec.~\ref{sec:ppo_or_loss}.

The group size $G=2$ was the largest setting that fit within the
available GPU memory under the implemented pipeline. During GRPO and
GRPO-OR training, the reward-model backbone remained on the CPU during
generation and policy updates and was moved to the GPU only for reward
scoring. This offloading procedure and all other within-family settings
were held fixed between GRPO and GRPO-OR. The smaller group size should
be treated as a hardware-constrained experimental setting rather than
as equivalent to the larger groups used in the original GRPO
experiments~\cite{DEEPSEEK2025}.

\begin{algorithm}[t]
\caption{Policy optimization for the GAE and group-relative variants}
\label{alg:rl_training}
\begin{algorithmic}[1]
\Require Frozen reward model $r_\phi$, frozen reference policy
$\piref$, trainable policy $\pinew$, prompt set $\mathcal{P}$,
method $m$
\State Initialize $\pinew$ from
\texttt{Llama-3.2-1B-Instruct}
\State Initialize $\piref$ as a frozen copy of the same checkpoint
\State Attach LoRA with $r=16$ and $\aL=32$ to
\texttt{q\_proj} and \texttt{v\_proj}
\For{rollout step $k=1,\dots,500$}
  \State Sample prompts $\{x_i\}_{i=1}^{B}$ with $B=4$
  \State Set $\piold$ to the rollout-policy snapshot for this batch
  \If{$m\in\{\text{PPO-clip},\text{PPO-OR}\}$}
    \State Generate one response per prompt with maximum length 128
    \State Record rollout-policy and reference-policy log-probabilities
    \State Record value predictions and terminal reward-model scores
    \State Construct reference-shaped rewards as in
           Sec.~\ref{sec:ppo_or_loss}
    \State Compute GAE advantages and return targets with
           $\gamma=1.0$ and $\lambda=0.95$
  \Else
    \State Generate $G=2$ responses per prompt with maximum length 64
    \State Record rollout-policy and reference-policy log-probabilities
    \State Compute response-level reward-model scores
    \State Compute group-relative advantages using Eq.~\eqref{eq:grpo_adv}
    \State Exclude groups with exactly zero reward variance
  \EndIf
  \For{optimization epoch $e=1,\dots,4$}
    \State Evaluate current-policy log-probabilities on valid response tokens
    \If{$m\in\{\text{PPO-OR},\text{GRPO-OR}\}$}
      \State $\mathcal{L}_{\mathrm{policy}}\gets L_{\mathrm{OR}}$
             using Eq.~\eqref{eq:l_or}
    \Else
      \State $\mathcal{L}_{\mathrm{policy}}\gets L_{\mathrm{clip}}$
             using Eq.~\eqref{eq:ppo_clip}
    \EndIf
    \If{$m\in\{\text{PPO-clip},\text{PPO-OR}\}$}
      \State Compute $\mathcal{L}_{\mathrm{value}}$
      \State $\mathcal{L}\gets
        \mathcal{L}_{\mathrm{policy}}
        +0.1\mathcal{L}_{\mathrm{value}}
        -0.01\mathcal{L}_{\mathrm{entropy}}$
    \Else
      \State Compute $L_{\mathrm{ref}}$ using
             Eq.~\eqref{eq:reference_penalty}
      \State $\mathcal{L}\gets
        \mathcal{L}_{\mathrm{policy}}
        +0.05L_{\mathrm{ref}}
        -0.01\mathcal{L}_{\mathrm{entropy}}$
    \EndIf
    \State Update LoRA parameters and, for GAE methods, the value head
           using AdamW
  \EndFor
\EndFor
\end{algorithmic}
\end{algorithm}

% ============================================================
\subsection{Main Results}
\label{sec:main_results}

Table~\ref{tab:main_results} summarizes the final training-time
reward-model scores and the principal optimization diagnostics for all
four methods. Each entry is based on seeds 42, 43, and 44 evaluated
with the same frozen reward model. We define the final score as the
last logged reward-model score at rollout step 500 for each seed and
report the mean and standard deviation across the three seeds. The
policy-loss and drift descriptors are approximate values read from the
terminal regions of the corresponding multi-seed mean curves and are
included only as descriptive diagnostics. Because the GAE and
group-relative configurations differ in advantage estimation, critic
usage, sampling procedure, and response length, the primary comparisons
are PPO-clip versus PPO-OR and GRPO versus GRPO-OR. Cross-family rankings
are reported for completeness but should not be interpreted as a
controlled factorial comparison.

\begin{table}[t]
\centering
\caption{Main results on Anthropic \texttt{hh-rlhf} across three seeds.
Final score is the mean and standard deviation of the last logged
training-time reward-model score at rollout step 500. Policy-loss and
drift entries are approximate terminal-region diagnostics read from the
multi-seed mean curves. Raw policy-loss magnitudes are objective-specific
and are not directly comparable across the clipped-surrogate and OR-MSE
objectives.}
\label{tab:main_results}
\small
\setlength{\tabcolsep}{3.5pt}
\renewcommand{\arraystretch}{1.12}
\begin{tabular}{@{}p{0.13\textwidth}p{0.16\textwidth}p{0.17\textwidth}p{0.31\textwidth}p{0.14\textwidth}@{}}
\toprule
\textbf{Method} &
\textbf{Advantage family} &
\textbf{Final score} &
\textbf{Terminal policy-loss behavior} &
\textbf{Mean $|\rho_t|$ drift} \\
\midrule
PPO-clip~\cite{SCHULMAN2017} &
GAE &
$0.195\pm0.110$ &
Near zero over the terminal region &
$\approx0.3$--$0.5$ \\
PPO-OR &
GAE &
$0.500\pm0.158$ &
Near zero over the terminal region &
$\approx0.3$--$0.5$ \\
GRPO~\cite{SHAO2024} &
Group-relative &
$0.488\pm0.080$ &
Large and highly variable; per-seed peaks reach the hundreds &
$\approx5$--$13$ \\
GRPO-OR &
Group-relative &
$0.457\pm0.027$ &
Near zero over the terminal region &
$\approx5$--$13$ \\
\bottomrule
\end{tabular}
\end{table}

\begin{figure}[t]
\centering
\includegraphics[width=0.9\linewidth]{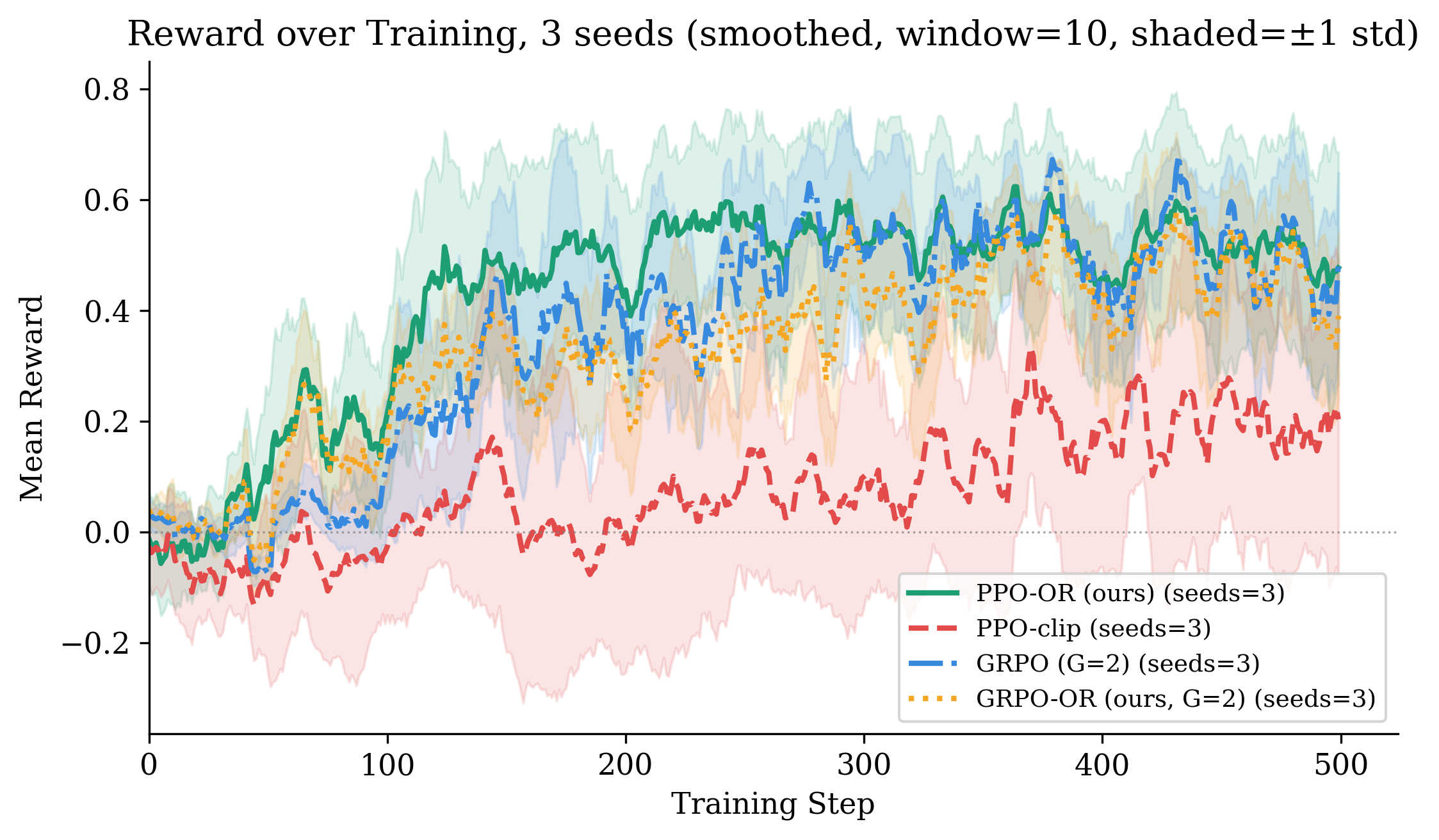}
\caption{Training-time reward-model score over 500 rollout steps for
all four methods, with three seeds per method. Curves are smoothed with
a moving window of 10 steps, and shaded regions denote $\pm1$ standard
deviation across seeds. The shared reward scale supports descriptive
within-experiment comparisons, but its numerical zero is not calibrated
as a success threshold.}
\label{fig:reward}
\end{figure}

Across the four configurations, PPO-OR has the highest observed mean
final score, $0.500\pm0.158$, followed by GRPO at $0.488\pm0.080$,
GRPO-OR at $0.457\pm0.027$, and PPO-clip at $0.195\pm0.110$.
Figure~\ref{fig:reward} shows that PPO-OR separates from PPO-clip during
training and reaches the highest mean terminal score, but it also has
the largest across-seed standard deviation. GRPO-OR has the smallest
reported standard deviation, although three seeds are insufficient to
support a general variance-reduction claim. These values are scores
from the same frozen reward model on training-time rollouts; they should
not be interpreted as calibrated utility values or held-out human
preference rates.

Within the GAE family, replacing the clipped surrogate with OR changes
the observed mean final score from $0.195$ to $0.500$, an absolute
difference of $0.305$ reward-score units. The observed across-seed
standard deviation changes from $0.110$ to $0.158$. Thus, the higher
observed mean is not accompanied by a smaller seed-to-seed spread. The
comparison shows a higher observed mean for PPO-OR than for the matched
PPO-clip configuration under the present model, dataset, training
budget, and hyperparameters; it does not by itself identify the
mechanism responsible for the difference or establish generalization
outside this setting.

Within the group-relative family, replacing the clipped surrogate with
OR changes the mean final score from $0.488$ to $0.457$, an absolute
difference of $-0.031$. At the same time, the across-seed standard
deviation decreases from $0.080$ to $0.027$, making the observed spread
approximately three times smaller. GRPO-OR therefore does not improve
the mean score in this matched comparison, but its three runs are more
tightly clustered. Because the estimate is based on only three seeds,
the reduced dispersion is reported as a descriptive observation rather
than a general stability guarantee.

\begin{figure}[t]
\centering
\includegraphics[width=0.9\linewidth]{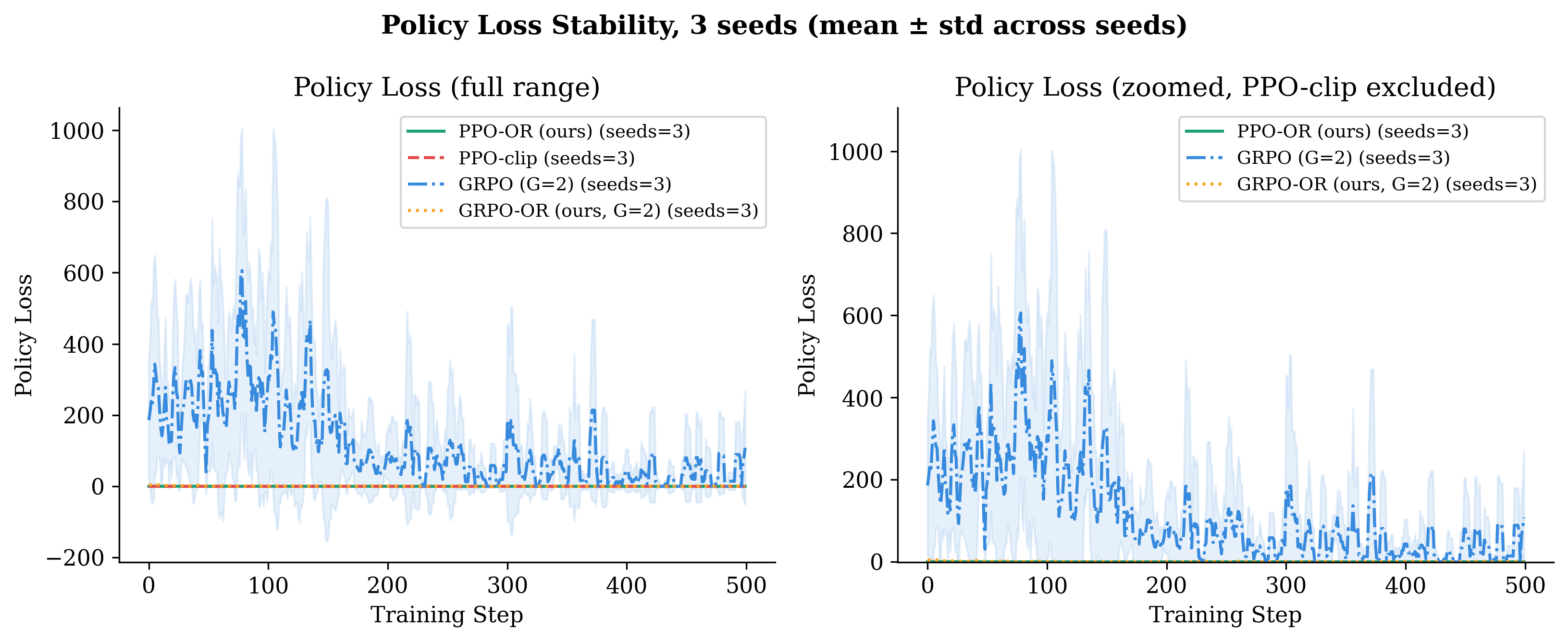}
\caption{Policy-loss traces over training for three seeds per method;
shaded regions denote $\pm1$ standard deviation. The left panel shows
the full plotted range, and the right panel enlarges the near-zero
region. GRPO has substantially larger and more variable logged loss
values than the other configurations. Absolute values across the
clipped-surrogate and OR-MSE objectives are not directly comparable,
so the figure is interpreted as an objective-specific optimization
diagnostic rather than a common performance scale.}
\label{fig:policy_loss}
\end{figure}

Figure~\ref{fig:policy_loss} adds an optimization view that is not
captured by the final reward-model scores. PPO-clip, PPO-OR, and
GRPO-OR remain near zero on their respective plotted loss scales,
whereas the GRPO trace is much larger and has a wide cross-seed band,
with individual peaks in the hundreds. This numerical difference does
not, by itself, show that one objective is better conditioned than
another because the clipped surrogate and OR loss have different
units, normalizations, and gradient coefficients. It nevertheless
shows that the strong final score obtained by GRPO coexists with a
large and variable objective-specific loss trace. Conversely, the
lower PPO-clip final score is not accompanied by an exploding raw loss.
Any explanation based on clipping activity or OR branch occupancy must
therefore be supported by the token-level diagnostics analyzed in
Sec.~\ref{sec:dynamics}, rather than inferred from raw loss magnitude
alone.

The policy-drift traces provide a second distinction between the two
advantage families. The GAE-based methods remain in an approximate
terminal range of $|\rho_t|\approx0.3$--$0.5$, whereas GRPO and GRPO-OR
remain mostly between approximately $5$ and $13$. For GRPO, the logged
mean reaches the display cap of 15 during roughly the first 130 rollout
steps before decreasing and continuing to oscillate at a substantially
higher level than the GAE methods. The same order of drift is observed
for GRPO-OR, indicating that replacing the group-relative clipped
surrogate with OR does not remove the large cross-family difference in
rollout-to-current policy displacement. These drift values are examined
in detail in Sec.~\ref{sec:dynamics} and should not be interpreted as
formal KL bounds.

\begin{figure}[t]
\centering
\includegraphics[width=0.9\linewidth]{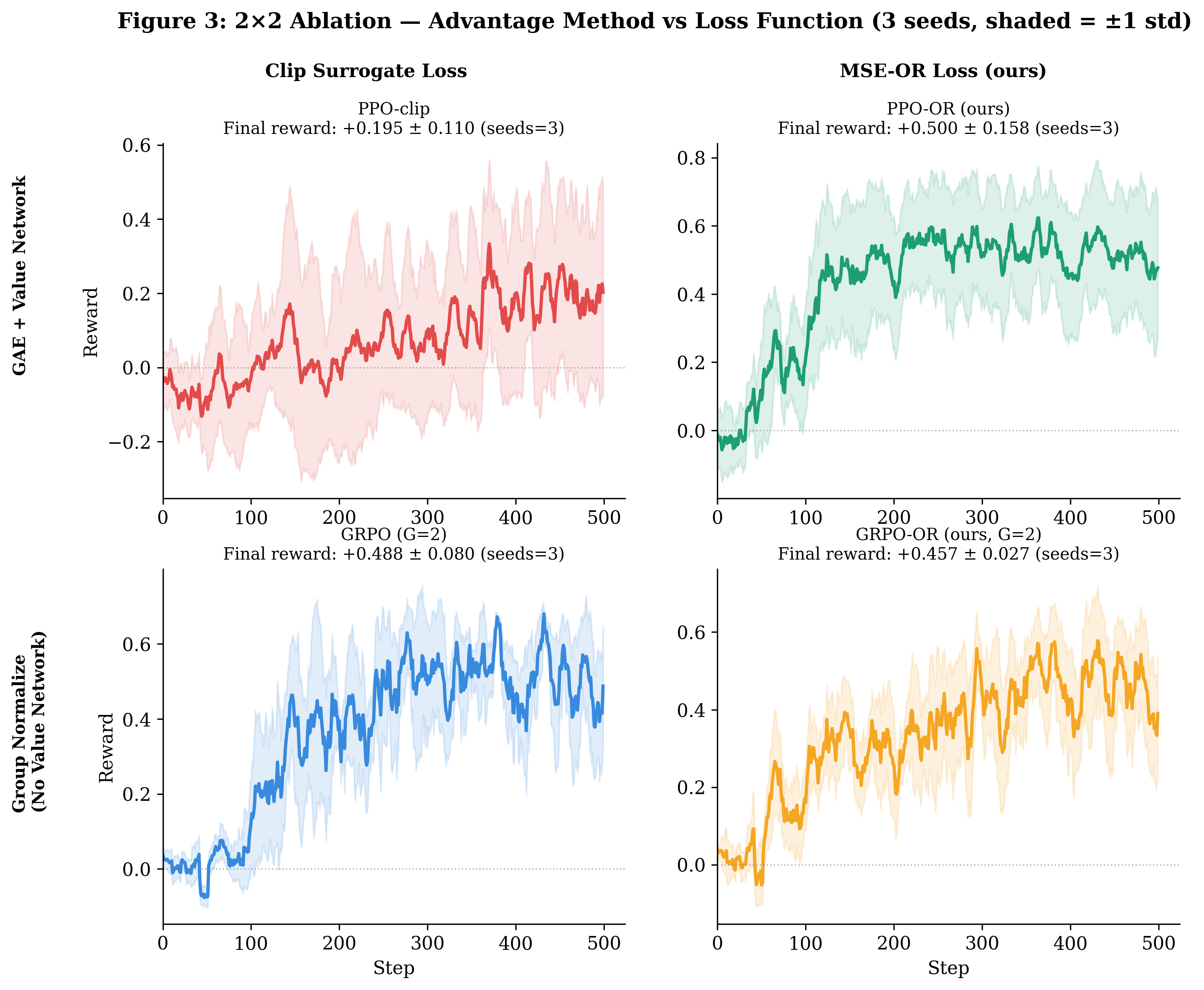}
\caption{Paired visualization of the policy-loss substitution within
each advantage family, using three seeds and $\pm1$ standard-deviation
bands. Under GAE, PPO-OR has a higher observed mean final score than
PPO-clip. Under group-relative advantages, GRPO-OR has a slightly lower
mean final score but a smaller observed cross-seed spread than GRPO.
The layout is not a controlled $2\times2$ factorial ablation because
the two advantage families also differ in critic usage, response length,
and sampling procedure.}
\label{fig:ablation_2x2}
\end{figure}

Figure~\ref{fig:ablation_2x2} summarizes the two matched substitutions
without treating them as a single factorial effect. The GAE comparison
shows a $+0.305$ difference in mean final score when the clipped
surrogate is replaced by OR, together with a $+0.048$ increase in the
observed standard deviation. The group-relative comparison shows a
$-0.031$ difference in mean final score and a $-0.053$ difference in
standard deviation. The two comparisons therefore exhibit different
observed patterns: higher mean score with greater dispersion under GAE,
and a similar but slightly lower mean with smaller dispersion under the
group-relative configuration. The present experiment does not isolate
an interaction between the loss and advantage estimator, because the
two families are not otherwise matched.

\raggedbottom
\setlength{\textfloatsep}{12pt plus 2pt minus 2pt}

\subsection{Training Dynamics}
\label{sec:dynamics}

We examine three complementary diagnostics: occupancy of the OR
overshoot branch, the magnitude of the logged OR targets, and the
sampled rollout-to-current policy displacement. These quantities
characterize different aspects of training and should not be interpreted
as interchangeable measures of convergence or stability. At rollout
step $k$, let $\mathcal{M}_k$ denote the valid generated-token positions.
For an OR-based method, we define the overshoot fraction as
\begin{equation}
  f_{\mathrm{over}}^{(k)}
  = \frac{1}{|\mathcal{M}_k|}
    \sum_{(b,t)\in\mathcal{M}_k}
    \mathbf{1}\!\left[
      \sigma_{b,t}\neq 0
      \ \text{and}\ 
      \sigma_{b,t}\rho_{b,t}>\aOR
    \right].
  \label{eq:overshoot_fraction}
\end{equation}
A token counted by Eq.~\eqref{eq:overshoot_fraction} lies beyond the
favorable OR margin and therefore contributes zero direct residual to
$L_{\mathrm{OR}}$ during that evaluation. A larger value consequently
means that a larger fraction of sampled tokens occupies the strict
zero-residual overshoot branch; it does not mean that OR is applying a
larger corrective gradient. Conversely, a value near zero does not by
itself establish policy convergence, because branch occupancy also
depends on the sampled responses, current advantage signs, and the
rollout-policy snapshot.

The implementation additionally logged a squared-target diagnostic,
shown as ``OR energy'' in the original training logs. To match the
quantity computed by the implementation, we define it as
\begin{equation}
  E_{\tau}^{(k)}
  = \frac{1}{|\mathcal{M}_k|}
    \sum_{(b,t)\in\mathcal{M}_k}
    \tau_{b,t}^{2}.
  \label{eq:or_target_magnitude}
\end{equation}
This is the mean squared magnitude of the logged OR target, not the
mean-squared residual $\left(\rho-\tau\right)^2$ used by the policy
loss. It is therefore reported only as an implementation-specific
target-magnitude diagnostic and not as a second loss or a convergence
certificate.

\begin{figure}[t]
\centering
\includegraphics[width=0.9\linewidth]{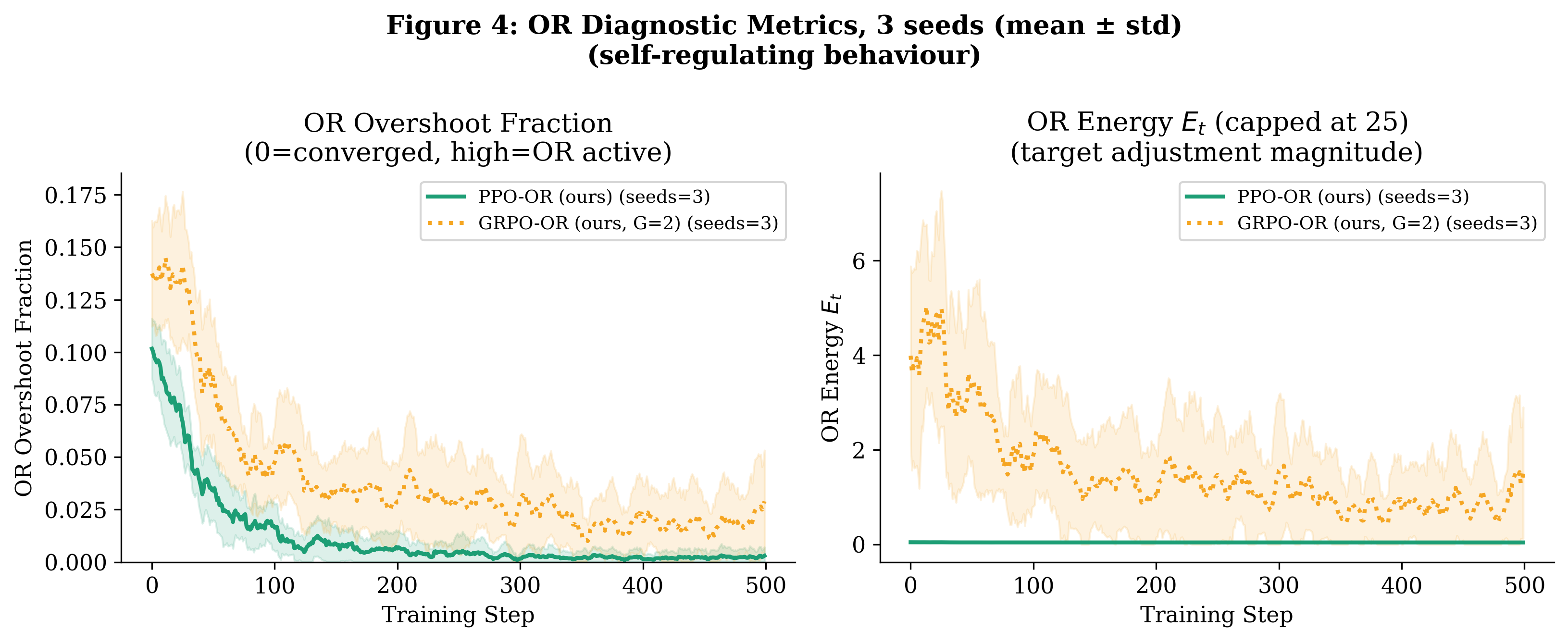}
\caption{OR diagnostics over training for PPO-OR and GRPO-OR, with
three seeds and $\pm1$ standard-deviation bands. Left: overshoot
fraction from Eq.~\eqref{eq:overshoot_fraction}. Tokens on this branch
have zero direct OR residual at that evaluation. Right: mean squared OR
target magnitude from Eq.~\eqref{eq:or_target_magnitude}, labeled ``OR
energy'' in the training logs and capped at 25 for display. Neither
quantity is a formal convergence measure or a bound on policy drift.}
\label{fig:or_diag}
\end{figure}

Figure~\ref{fig:or_diag} shows different branch-occupancy patterns for
the two OR variants. For PPO-OR, the mean overshoot fraction decreases
from approximately $0.10$ near the start of training to approximately
$0.005$ by rollout steps $150$--$200$ and remains low thereafter. For
GRPO-OR, it begins near $0.14$, declines more gradually, and remains
near $0.03$ with visible cross-seed variation through step 500. The
squared-target diagnostic follows a similar descriptive separation:
it remains close to zero on the plotted scale for PPO-OR, whereas for
GRPO-OR it decreases from approximately $4$--$5$ to a noisy terminal
range of approximately $1$--$1.5$. These observations show that the two
OR configurations spend different fractions of training on the strict
overshoot branch and produce different target magnitudes. They do not,
by themselves, establish that either policy has converged or identify
the advantage estimator as the cause, because the GAE and group-relative
configurations also differ in critic usage, response length, and group
sampling.

The objective-specific loss traces in Fig.~\ref{fig:policy_loss}
provide a related but distinct view. GRPO-OR's logged OR loss approaches
a near-zero terminal region, while the GRPO clipped-surrogate trace
remains large and highly variable, with individual peaks in the
hundreds. Because these methods optimize losses with different units,
normalizations, and gradient coefficients, this numerical contrast
cannot be interpreted as a common-scale reduction in policy loss or as
proof that OR removes a particular instability mechanism. It does show
that the OR residual remains small under the GRPO-OR objective while the
matched GRPO run retains a large objective-specific trace. Under the OR
loss, a small value can arise because active tokens are close to their
signed margins, because tokens occupy the zero-residual branch, or from
a combination of both; Fig.~\ref{fig:or_diag} is needed to distinguish
these cases descriptively.

For the displacement diagnostic, we use the valid-token mean absolute
log-ratio between the current policy and the rollout policy,
\begin{equation}
  \begin{aligned}
    D_{\mathrm{roll}}^{(k)}
    &= \frac{1}{|\mathcal{M}_k|}
       \sum_{(b,t)\in\mathcal{M}_k}
       \bigl|\rho_{b,t}\bigr|, \\
    \rho_{b,t}
    &= \log\pinew(a_{b,t}\mid s_{b,t})
       - \log\piold(a_{b,t}\mid s_{b,t}).
  \end{aligned}
  \label{eq:rollout_drift}
\end{equation}
The plotted value is capped at 15 for readability. This sampled
absolute log-ratio is not a categorical KL divergence and does not
measure displacement from the frozen reference policy.

\begin{figure}[t]
\centering
\includegraphics[width=0.9\linewidth]{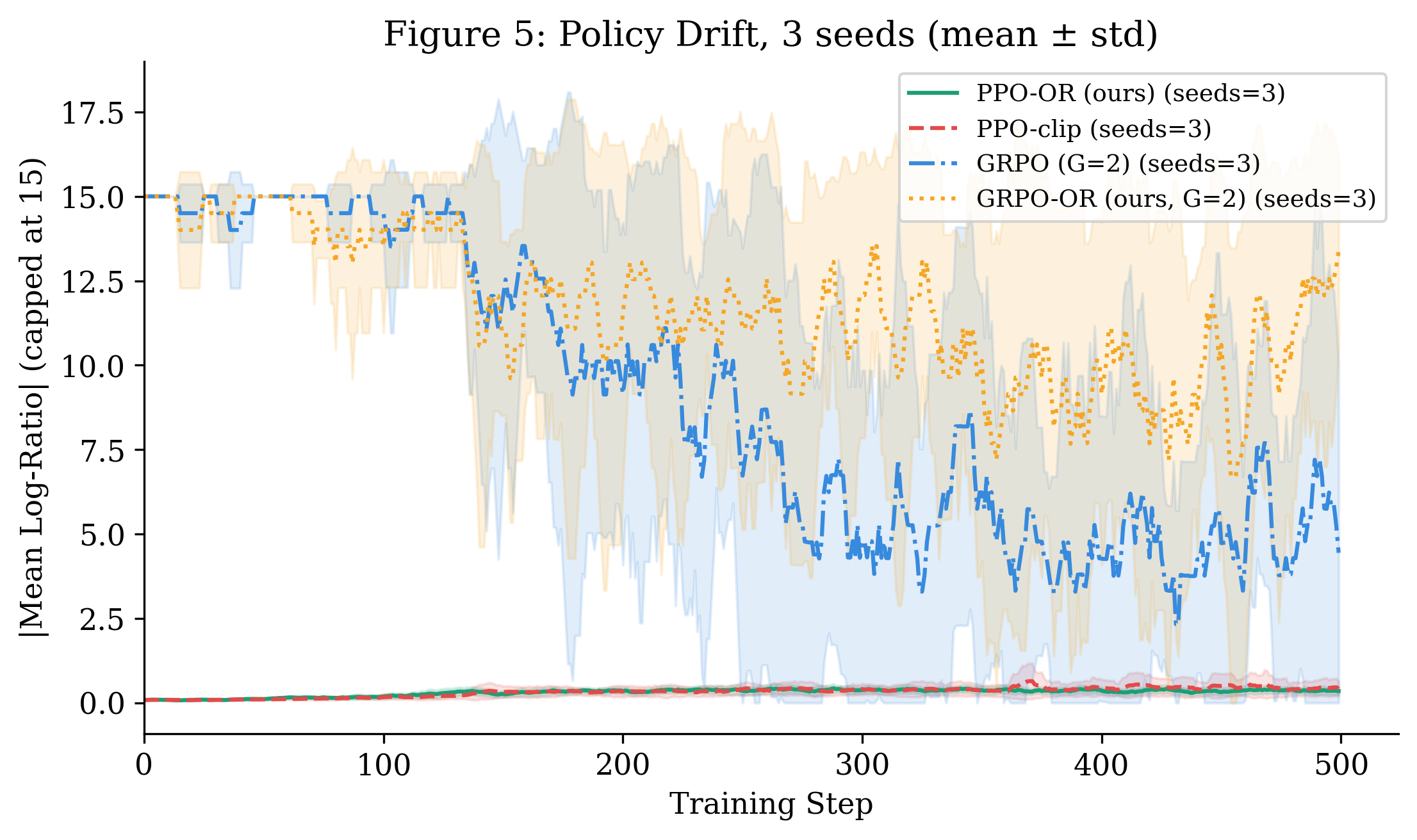}
\caption{Sampled rollout-to-current policy displacement from
Eq.~\eqref{eq:rollout_drift}, averaged over valid generated tokens for
three seeds. Curves show multi-seed means with $\pm1$ standard-deviation
bands and are capped at 15 for display; values at the cap may represent
larger uncapped displacement. The metric is an absolute sampled
log-ratio, not a KL divergence or a trust-region bound.}
\label{fig:drift}
\end{figure}

Figure~\ref{fig:drift} separates the GAE and group-relative
configurations more sharply than the final reward-model scores. PPO-clip
and PPO-OR remain below approximately $0.5$ throughout most of training
and end in the approximate range $0.3$--$0.5$. GRPO reaches the display
cap during roughly the first 130 rollout steps and subsequently
oscillates mostly between approximately $5$ and $13$. GRPO-OR remains
on the same order and exceeds GRPO during portions of training rather
than consistently reducing the displacement. Thus, replacing the
clipped surrogate with OR does not remove the large rollout-to-current
policy movement observed in the group-relative configuration.

This result is consistent with the scope of the analytical properties
in Sec.~\ref{sec:theory}. Proposition~\ref{prop:diff} states that a token
on the favorable saturated branch contributes zero direct OR residual;
it does not impose a hard trust region or bound the cumulative optimizer
update, and the active-branch derivative remains unbounded in the
adverse direction. The realized update also combines gradients from
other tokens, entropy and reference-policy terms, and, for PPO-OR, the
value objective. Moreover, rollout batches and advantage signs change
across iterations. The training dynamics therefore distinguish
sample-level one-sided saturation from policy-level proximity: OR
changes when individual token losses become inactive, but the present
results do not show that it constrains aggregate policy displacement.

% ============================================================
\subsection{Runtime and Memory Usage}
\label{sec:compute}

We report wall-clock and GPU-memory measurements for the implemented
training pipelines rather than asymptotic computational complexity.
Table~\ref{tab:compute_efficiency} summarizes the measured runtime of
one representative run for each method. The logged step time covers the
timed training-step region, whereas total wall-clock time also includes
model initialization, reward-model transfers, logging, checkpointing,
and other run-level overhead. The two timing columns should therefore
not be multiplied directly. Because these measurements were obtained
from one run per method rather than averaged across seeds, they should
be interpreted as descriptive measurements of this implementation, not
as a hardware-independent ranking of the four objectives.

\begin{table}[t]
\centering
\caption{Runtime and steady-state GPU-memory measurements for the
implemented training pipelines. Step times were recorded from one
representative run per method; total runtime includes run-level overhead.
Steady-state VRAM values summarize the ranges observed in
Fig.~\ref{fig:vram}.}
\label{tab:compute_efficiency}
\small
\setlength{\tabcolsep}{3.5pt}
\renewcommand{\arraystretch}{1.12}
\begin{tabular}{@{}lccc@{}}
\toprule
\textbf{Method} &
\textbf{Logged time/step} &
\textbf{Total time} &
\textbf{Steady-state VRAM} \\
\midrule
PPO-OR  & $1.3$ s & $22$ min & $\approx5.5$--$6$ GB \\
PPO-clip & $4.4$ s & $42$ min & $\approx5.5$--$6$ GB \\
GRPO     & $8.9$ s & $77$ min & $\approx4$ GB \\
GRPO-OR  & $4.6$ s & $49$ min & $\approx4$ GB \\
\bottomrule
\end{tabular}
\end{table}

Under this implementation, PPO-OR had the lowest measured step time and
total runtime, followed by PPO-clip and GRPO-OR, while GRPO required the
longest run. The group-relative pipelines generate two responses per
prompt and repeatedly move the reward-model backbone between CPU and GPU,
which contributes additional generation and transfer overhead. The
substantial timing differences between the clipped and OR variants may
also reflect implementation details outside the scalar policy objective,
including tensor construction, masking, logging, and the numerical
behavior of each training path. We therefore do not claim that replacing
the clipped surrogate with OR necessarily yields the same runtime ordering
in other implementations.

\begin{figure}[t]
\centering
\includegraphics[width=0.9\linewidth]{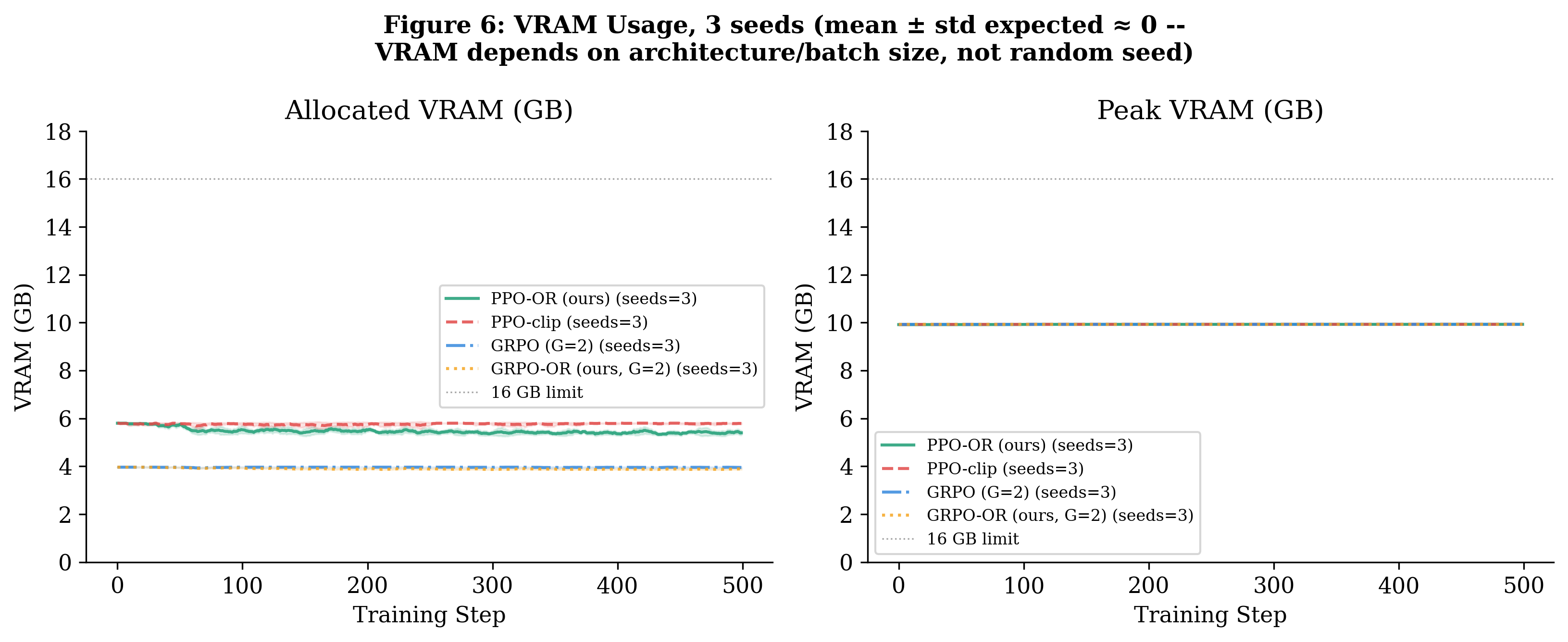}
\caption{Allocated GPU memory over 500 rollout steps for three seeds per
method. Curves show the seed mean and shaded regions show $\pm1$ standard
deviation. All four pipelines remain below the reported 16.7~GB device
budget. The narrow bands indicate limited run-to-run variation in the
logged memory allocation under the fixed architecture and batch settings.}
\label{fig:vram}
\end{figure}

Figure~\ref{fig:vram} shows a short-lived peak of approximately 10~GB
during initialization, when the policy, reference, and reward-model
components are loaded before the steady-state execution pattern is
established. During training, PPO-OR and PPO-clip allocate approximately
5.5--6~GB, whereas GRPO and GRPO-OR allocate approximately 4~GB. The
lower steady-state GPU allocation of the group-relative pipelines is due
to the implemented CPU offloading of the reward-model backbone and their
shorter 64-token response cap; it should not be interpreted as an
intrinsic memory advantage of group-relative optimization. The GAE-based
methods retain the reward-model components on the GPU and generate up to
128 response tokens.

The available memory constrained the group-relative experiments to
$G=2$. Under the implemented pipeline, this was the largest group size
that could be run reliably together with four prompts per rollout batch.
An engineering extrapolation from the observed allocations suggested that
$G=4$ could approach approximately 13~GB of peak usage and might be
feasible with additional memory optimization, whereas $G=8$ could require
approximately 18--20~GB and would exceed the available budget. These are
planning estimates rather than direct end-to-end measurements and do not
include every possible memory optimization. Larger groups would change
the group-relative advantage estimates, but whether they would reduce the
policy-displacement behavior reported in Sec.~\ref{sec:dynamics} remains
an empirical question.

% ============================================================
\section{Discussion}
\label{sec:discussion}

\subsection{Interpreting the Observed Effects of OR}

The central result is not that OR uniformly stabilizes policy
optimization, but that its observed effect depends on the surrounding
advantage-estimation pipeline and on the diagnostic being examined. In
the matched GAE comparison, PPO-OR has a higher observed mean final
training-time reward-model score than PPO-clip,
$0.500\pm0.158$ versus $0.195\pm0.110$, an absolute difference of
$0.305$ score units; the observed cross-seed standard deviation is also
larger by $0.048$. In the matched group-relative comparison, replacing the
clipped surrogate with OR changes the mean score from $0.488\pm0.080$
to $0.457\pm0.027$, a difference of $-0.031$, while the observed
standard deviation decreases by $0.053$. GRPO-OR also produces a small
OR residual over the terminal region, whereas the matched GRPO run has
a large and variable clipped-surrogate trace. Because the two methods
optimize objectives with different units and gradient coefficients,
this contrast is an objective-specific training observation rather than
a common-scale reduction in policy loss. Most importantly, GRPO-OR does
not consistently reduce the large rollout-to-current displacement of
the group-relative configuration.

These results separate three questions that are often grouped under the
single label of training stability: whether the policy obtains a high
score under the learned reward model, whether the optimized scalar
objective has a well-behaved numerical trace, and whether the updated
policy remains close to the rollout policy. The three quantities do not
move together in our experiments. PPO-clip has a near-zero terminal
loss trace but the lowest mean reward-model score. GRPO has a strong
mean final score despite a large objective-specific trace and much
larger sampled policy displacement than the GAE methods. GRPO-OR has a
small OR residual but remains on the same order of displacement as
GRPO. Reward, raw loss, and policy displacement therefore provide
complementary rather than interchangeable evidence.

This interpretation also clarifies the relation of PPO-OR to prior
policy-optimization methods. PPO and GRPO use a clipped probability-ratio
surrogate~\cite{SCHULMAN2017,SHAO2024}, whereas OR originated as a
target-reset procedure for MSE-based classification~\cite{TYAGI2024}.
The present work transfers that target-reset construction to a
one-sided squared margin in token-level log-ratio space. The experiments
show that this substitution can change reward-model scores and
objective-specific dynamics under the evaluated conditions, but they do
not support treating OR as a replacement for every stabilizing role
played by clipping, a critic, reference regularization, or the
advantage estimator. In particular, the smooth one-sided saturation of
the scalar OR loss should not be interpreted as a policy-level trust
region.

\subsection{Interpreting the Persistent Policy Displacement}

The group-relative displacement result is consistent with the scope of
the Method section. Proposition~\ref{prop:diff} establishes that a token
on the favorable saturated branch contributes zero direct OR residual
during that evaluation. It does not constrain the optimizer step,
bound the adverse branch, reverse displacement accumulated in earlier
updates, or guarantee proximity to either $\piold$ or $\piref$. The
active-branch derivative can grow without bound when the sampled
log-ratio moves far in the adverse direction, and the realized minibatch
update combines contributions from many tokens together with entropy
and reference-policy terms. Four optimization epochs are also performed
for each rollout batch before a new rollout-policy snapshot is formed.
Consequently, sample-level saturation and small OR residuals can coexist
with large aggregate movement from the rollout policy.

Changing batches and advantage signs add another distinction between
the scalar analysis and the complete training process. The favorable
branch for a sampled token is defined using the current estimated
advantage sign, while later rollouts may contain different tokens,
states, and signs. A token that is inactive in one evaluation therefore
does not remain permanently constrained in subsequent updates. This
observation explains why the Method does not predict a cumulative drift
bound, but it does not by itself explain why the group-relative runs
reach displacement values an order of magnitude larger than the GAE
runs.

Within the group-relative family, GRPO and GRPO-OR share the same
group size, response cap, sampling procedure, optimization epochs, and
reference-policy treatment. The matched comparison therefore supports
one limited conclusion: replacing the clipped surrogate with the present
OR loss is insufficient to remove the observed displacement under this
configuration. The mechanisms responsible for the larger
group-relative displacement, and the factors that might change it, are
left to Sec.~\ref{sec:open_questions} rather than inferred from the
current runs.

\subsection{Limitations and Scope}
\label{sec:limitations}

The statistical evidence is limited to three seeds per method. The
reported means and standard deviations describe the observed runs but do
not support significance testing, reliable ranking of variance, or
claims that the same ordering would persist under additional seeds. This
qualification applies both to the larger observed spread of PPO-OR and
to the smaller observed spread of GRPO-OR. The terminal policy-loss and
drift entries in Table~\ref{tab:main_results} are approximate ranges
read from multi-seed mean curves rather than statistics computed from
exported per-seed terminal values. They support the reported
order-of-magnitude contrast, but not fine-grained numerical comparison.

Evaluation relies on training-time scores from one frozen reward model,
not on independent human judgments or a fixed held-out preference
benchmark. The Bradley--Terry objective identifies score differences but
does not calibrate an absolute zero, and the reward model's reported
60--61\% pairwise accuracy indicates a noisy preference signal. Sharing
this scorer across methods removes reward-model variation between runs,
but it does not establish that higher proxy scores correspond to better
human-preferred behavior or that scorer errors affect every policy
distribution equally. More generally, optimizing an imperfect reward
model can separate proxy reward from the intended preference objective
~\cite{GAO2023}. External validity is further limited by the use of one
instruction-tuned 1B model and one dataset. The study does not include a
task-specific SFT stage, and \texttt{hh-rlhf} supplies both reward-model
pairs and policy-optimization prompts; possible prompt overlap and reuse
of the same source may favor the training distribution. The findings
therefore characterize this constrained post-training setup rather than
a general ranking of PPO, GRPO, and OR-based methods. 

The matched comparisons use shared rather than independently tuned
hyperparameters. In the GAE family, the common nominal setting
$\epsilon=\aOR=0.2$ does not imply identical saturation boundaries:
PPO clipping corresponds to the asymmetric log-ratio interval
$[\log(0.8),\log(1.2)]\approx[-0.223,+0.182]$, whereas PPO-OR uses
symmetric margins at $[-0.2,+0.2]$. No sensitivity analysis was
performed for the OR margin, value-head warmup, value clipping,
method-specific learning rates, or optimization epochs. PPO-clip's lower
observed score in this experiment should therefore not be interpreted as
a general inability to learn. The group-relative results are likewise
restricted by the hardware-constrained choice $G=2$, for which the
standard-deviation convention is especially consequential. Larger
groups may alter reward diversity, advantage magnitudes, and update
dynamics, but the present runs do not show whether they would reduce
policy displacement. Cross-family comparisons are additionally
confounded by critic usage, reference-regularization placement, response
length, sampling procedure, and CPU offloading; the strongest evidence
remains the two within-family policy-loss substitutions.

Finally, runtime measurements come from one representative run per
method and reflect the particular implementation, including model
placement, reward-model transfers, masking, and logging. They do not
establish an intrinsic computational advantage of the scalar OR
objective. The larger-group memory values are engineering projections
rather than measured end-to-end runs. The questions left unresolved by
these limitations and the experiments needed to address them are
presented separately in Secs.~\ref{sec:open_questions}
and~\ref{sec:future}.

% ============================================================
\section{Open Questions Raised by the Study}
\label{sec:open_questions}

The experiments leave three questions unresolved.

\emph{What drives the large rollout-to-current policy displacement in
the group-relative runs?} GRPO and GRPO-OR both move much farther from
the rollout policy than the GAE-based methods, even though the GRPO-OR
residual approaches zero and its overshoot fraction declines during
training. The present study cannot isolate the effects of the small
response group ($G=2$), critic removal, the shorter response cap,
repeated optimization epochs, or the placement of reference-policy
regularization.

\emph{Do the OR diagnostics have predictive value?} The overshoot
fraction measures how often tokens lie beyond the favorable margin,
while the logged ``OR energy'' measures squared target magnitude. A
declining overshoot fraction and a near-zero OR residual describe the
state of the OR objective, but neither establishes policy convergence.
It remains unknown whether either diagnostic predicts subsequent reward
change, policy displacement, or optimization failure.

\emph{Why do the observed effects of OR differ between the two matched
comparisons?} Under GAE, PPO-OR produces a higher observed mean
reward-model score but a larger observed seed spread. Under
group-relative advantages, GRPO-OR produces a smaller observed seed
spread and a near-zero terminal residual without a higher observed mean
score. The current evidence cannot determine whether this contrast is
driven by advantage construction, interaction with the OR margin,
reward-model noise, or other family-specific differences in the
training pipelines.

% ============================================================
\section{Future Directions}
\label{sec:future}

The highest-priority experiment is a controlled group-size study with
$G\in\{2,4,8\}$ while holding the model, prompt distribution, response
length, sampled-token budget, update epochs, and reference-policy
coefficient fixed. The study should report uncapped rollout-to-current
displacement, displacement from the frozen reference policy,
reward-model score, and OR branch occupancy. If larger groups do not
reduce displacement, stronger adaptive reference regularization or an
explicit two-sided displacement penalty should be compared against
standard KL-based controls. The GAE comparison should be repeated with more seeds and with both
shared and method-specific hyperparameter tuning. The most relevant
factors are value-head warmup, the OR margin, and the mismatch between
PPO's asymmetric log-ratio clipping boundaries and the symmetric OR
margins used here. An asymmetric OR variant would test whether the
current margin choice contributes to the observed mean--dispersion
trade-off.

A second extension is to retain advantage-magnitude information on
active tokens while preserving the OR zero-residual branch:
\[
L_{\mathrm{OR}}^{w}
=
\frac{1}{N_{\mathcal{M}}}
\sum_{(b,t)\in\mathcal{M}}
\left|\hat{A}_{b,t}\right|
\left(
\rho_{b,t}-\sg[\tau_{b,t}]
\right)^2.
\]
This objective is proposed rather than evaluated here. Raw, normalized,
and clipped magnitude weights should be compared to determine whether
restoring magnitude information improves reward or displacement, or
instead reintroduces sensitivity to noisy advantage estimates. Finally, external validity requires disjoint reward-model and RL prompt
sets, checkpoint selection on held-out validation data, and final
evaluation on separate prompts using independent preference judgments.
The matched PPO-clip versus PPO-OR and GRPO versus GRPO-OR comparisons
should then be repeated with a stronger reward model, a task-specific
SFT checkpoint, and larger policies. A staged study at approximately
1B, 3B, and 7B--8B scale is more defensible than extrapolating directly
from the present 1B experiment to very large models.

\section{Conclusion}
\label{sec:conclusion}

We introduced PPO-OR and GRPO-OR, which replace the PPO-style clipped
policy term with a continuously differentiable, one-sided squared margin
in rollout-relative token log-ratio space while retaining the advantage
estimator and remaining optimization components of their parent
pipelines. The observed effect depends on the advantage family. In the
matched GAE comparison, PPO-OR achieves a higher mean final
training-time reward-model score than PPO-clip by $0.305$ score units,
although its observed across-seed spread is also larger. In the matched
group-relative comparison, GRPO-OR does not achieve a higher mean score,
but it shows a smaller observed seed spread, a near-zero terminal OR
residual, and a declining overshoot fraction. These objective-specific
diagnostics should not be interpreted as directly comparable loss
magnitudes or as evidence of policy convergence.The principal qualification is that local OR saturation does not by itself control aggregate policy movement. Both group-relative methods
retain substantially larger rollout-to-current log-ratio displacement
than the GAE-based methods, and replacing the clipped GRPO objective with
OR does not remove that displacement. The supported conclusion is
therefore narrower than a general claim of superiority: OR changes the
sample-level saturation geometry and can alter reward-score and
optimization behavior, but it is not a trust-region mechanism and does
not replace the roles of advantage estimation, critic learning, or
reference-policy regularization. Whether the observed family-dependent
effects persist under larger response groups, stronger reward models,
held-out preference evaluation, and method-specific tuning remains open.

\bibliographystyle{unsrt}
\bibliography{arxiv}

@article{WILLIAMS1992,
  title   = {Simple Statistical Gradient-Following Algorithms for Connectionist Reinforcement Learning},
  author  = {Williams, Ronald J.},
  journal = {Machine Learning},
  volume  = {8},
  pages   = {229--256},
  year    = {1992},
  doi     = {10.1007/BF00992696}
}

@inproceedings{SUTTON1999,
  title     = {Policy Gradient Methods for Reinforcement Learning with Function Approximation},
  author    = {Sutton, Richard S. and McAllester, David A. and Singh, Satinder P. and Mansour, Yishay},
  booktitle = {Advances in Neural Information Processing Systems},
  volume    = {12},
  pages     = {1057--1063},
  year      = {1999}
}

@inproceedings{SCHULMAN2015,
  title     = {Trust Region Policy Optimization},
  author    = {Schulman, John and Levine, Sergey and Moritz, Philipp and Jordan, Michael I. and Abbeel, Pieter},
  booktitle = {Proceedings of the 32nd International Conference on Machine Learning},
  series    = {Proceedings of Machine Learning Research},
  volume    = {37},
  pages     = {1889--1897},
  year      = {2015},
  publisher = {PMLR},
  url       = {https://proceedings.mlr.press/v37/schulman15.html}
}

@article{SCHULMAN2016,
  title         = {High-Dimensional Continuous Control Using Generalized Advantage Estimation},
  author        = {Schulman, John and Moritz, Philipp and Levine, Sergey and Jordan, Michael I. and Abbeel, Pieter},
  journal       = {International Conference on Learning Representations},
  year          = {2016},
  eprint        = {1506.02438},
  archivePrefix = {arXiv},
  primaryClass  = {cs.LG},
  url           = {https://arxiv.org/abs/1506.02438}
}

@article{SCHULMAN2017,
  title         = {Proximal Policy Optimization Algorithms},
  author        = {Schulman, John and Wolski, Filip and Dhariwal, Prafulla and Radford, Alec and Klimov, Oleg},
  journal       = {arXiv preprint arXiv:1707.06347},
  year          = {2017},
  eprint        = {1707.06347},
  archivePrefix = {arXiv},
  primaryClass  = {cs.LG},
  url           = {https://arxiv.org/abs/1707.06347}
}

@inproceedings{CHRISTIANO2017,
  title     = {Deep Reinforcement Learning from Human Preferences},
  author    = {Christiano, Paul F. and Leike, Jan and Brown, Tom and Martic, Miljan and Legg, Shane and Amodei, Dario},
  booktitle = {Advances in Neural Information Processing Systems},
  volume    = {30},
  year      = {2017},
  url       = {https://proceedings.neurips.cc/paper/2017/hash/d5e2c0adad503c91f91df240d0cd4e49-Abstract.html}
}

@inproceedings{STIENNON2020,
  title     = {Learning to Summarize with Human Feedback},
  author    = {Stiennon, Nisan and Ouyang, Long and Wu, Jeffrey and Ziegler, Daniel and Lowe, Ryan and Voss, Chelsea and Radford, Alec and Amodei, Dario and Christiano, Paul F.},
  booktitle = {Advances in Neural Information Processing Systems},
  volume    = {33},
  year      = {2020},
  url       = {https://proceedings.neurips.cc/paper/2020/hash/1f89885d556929e98d3ef9b86448f951-Abstract.html}
}

@inproceedings{OUYANG2022,
  title     = {Training Language Models to Follow Instructions with Human Feedback},
  author    = {Ouyang, Long and Wu, Jeffrey and Jiang, Xu and Almeida, Diogo and Wainwright, Carroll and Mishkin, Pamela and Zhang, Chong and Agarwal, Sandhini and Slama, Katarina and Ray, Alex and others},
  booktitle = {Advances in Neural Information Processing Systems},
  volume    = {35},
  year      = {2022},
  url       = {https://proceedings.neurips.cc/paper_files/paper/2022/hash/b1efde53be364a73914f58805a001731-Abstract.html}
}

@article{BAIETAL2022,
  title         = {Training a Helpful and Harmless Assistant with Reinforcement Learning from Human Feedback},
  author        = {Bai, Yuntao and Jones, Andy and Ndousse, Kamal and Askell, Amanda and Chen, Anna and DasSarma, Nova and Drain, Dawn and Fort, Stanislav and Ganguli, Deep and Henighan, Tom and others},
  journal       = {arXiv preprint arXiv:2204.05862},
  year          = {2022},
  eprint        = {2204.05862},
  archivePrefix = {arXiv},
  primaryClass  = {cs.CL},
  url           = {https://arxiv.org/abs/2204.05862}
}

@article{BRADLEY1952,
  title   = {Rank Analysis of Incomplete Block Designs: I. The Method of Paired Comparisons},
  author  = {Bradley, Ralph Allan and Terry, Milton E.},
  journal = {Biometrika},
  volume  = {39},
  number  = {3/4},
  pages   = {324--345},
  year    = {1952},
  doi     = {10.1093/biomet/39.3-4.324}
}

@inproceedings{ENGSTROM2020,
  title     = {Implementation Matters in Deep RL: A Case Study on PPO and TRPO},
  author    = {Engstrom, Logan and Ilyas, Andrew and Santurkar, Shibani and Tsipras, Dimitris and Janoos, Firdaus and Rudolph, Larry and Madry, Aleksander},
  booktitle = {International Conference on Learning Representations},
  year      = {2020},
  url       = {https://openreview.net/forum?id=r1etN1rtPB}
}

@inproceedings{ILYAS2020,
  title     = {A Closer Look at Deep Policy Gradients},
  author    = {Ilyas, Andrew and Engstrom, Logan and Santurkar, Shibani and Tsipras, Dimitris and Janoos, Firdaus and Rudolph, Larry and Madry, Aleksander},
  booktitle = {International Conference on Learning Representations},
  year      = {2020},
  url       = {https://openreview.net/forum?id=ryxdEkHtPS}
}

@inproceedings{WANG2020,
  title     = {Truly Proximal Policy Optimization},
  author    = {Wang, Yuhui and He, Hao and Tan, Xiaoyang},
  booktitle = {Proceedings of the 35th Uncertainty in Artificial Intelligence Conference},
  series    = {Proceedings of Machine Learning Research},
  volume    = {115},
  pages     = {113--122},
  year      = {2020},
  publisher = {PMLR},
  url       = {https://proceedings.mlr.press/v115/wang20b.html}
}

@inproceedings{GARG2021,
  title     = {On Proximal Policy Optimization's Heavy-Tailed Gradients},
  author    = {Garg, Saurabh and Zhanson, Joshua and Parisotto, Emilio and Prasad, Adarsh and Kolter, J. Zico and Lipton, Zachary Chase and Balakrishnan, Sivaraman and Salakhutdinov, Ruslan and Ravikumar, Pradeep},
  booktitle = {Proceedings of the 38th International Conference on Machine Learning},
  series    = {Proceedings of Machine Learning Research},
  volume    = {139},
  pages     = {3610--3619},
  year      = {2021},
  publisher = {PMLR},
  url       = {https://proceedings.mlr.press/v139/garg21b.html}
}

@inproceedings{JIN2024,
  title     = {On Stationary Point Convergence of PPO-Clip},
  author    = {Jin, Ruinan and Li, Shuai and Wang, Baoxiang},
  booktitle = {International Conference on Learning Representations},
  year      = {2024},
  url       = {https://openreview.net/forum?id=uznKlCpWjV}
}

@inproceedings{XIE2025,
  title     = {Simple Policy Optimization},
  author    = {Xie, Zhengpeng and Zhang, Qiang and Yang, Fan and Hutter, Marco and Xu, Renjing},
  booktitle = {Proceedings of the 42nd International Conference on Machine Learning},
  series    = {Proceedings of Machine Learning Research},
  volume    = {267},
  pages     = {68813--68824},
  year      = {2025},
  publisher = {PMLR},
  url       = {https://proceedings.mlr.press/v267/xie25m.html}
}

@article{SHAO2024,
  title         = {{DeepSeekMath}: Pushing the Limits of Mathematical Reasoning in Open Language Models},
  author        = {Shao, Zhihong and Wang, Peiyi and Zhu, Qihao and Xu, Runxin and Song, Junxiao and Bi, Xiao and Zhang, Haowei and Zhang, Mingchuan and Li, Y. K. and Wu, Y. and Guo, Daya},
  journal       = {arXiv preprint arXiv:2402.03300},
  year          = {2024},
  eprint        = {2402.03300},
  archivePrefix = {arXiv},
  primaryClass  = {cs.CL},
  url           = {https://arxiv.org/abs/2402.03300}
}

@article{DEEPSEEK2025,
  title         = {{DeepSeek-R1}: Incentivizing Reasoning Capability in {LLMs} via Reinforcement Learning},
  author        = {{DeepSeek-AI} and Guo, Daya and Yang, Dejian and Zhang, Haowei and Song, Junxiao and Zhang, Ruoyu and Xu, Runxin and Zhu, Qihao and Ma, Shirong and Wang, Peiyi and others},
  journal       = {arXiv preprint arXiv:2501.12948},
  year          = {2025},
  eprint        = {2501.12948},
  archivePrefix = {arXiv},
  primaryClass  = {cs.CL},
  url           = {https://arxiv.org/abs/2501.12948}
}

@article{BYTEDANCE2025,
  title         = {{DAPO}: An Open-Source {LLM} Reinforcement Learning System at Scale},
  author        = {Yu, Qiying and Zhang, Zheng and Zhu, Ruofei and Yuan, Yufeng and Zuo, Xiaochen and Yue, Yu and Dai, Weinan and Fan, Tiantian and Liu, Gaohong and others},
  journal       = {arXiv preprint arXiv:2503.14476},
  year          = {2025},
  eprint        = {2503.14476},
  archivePrefix = {arXiv},
  primaryClass  = {cs.LG},
  url           = {https://arxiv.org/abs/2503.14476}
}

@article{DRGRPO,
  title         = {Understanding R1-Zero-Like Training: A Critical Perspective},
  author        = {Liu, Zichen and Chen, Changyu and Li, Wenjun and Qi, Penghui and Pang, Tianyu and Du, Chao and Lee, Wee Sun and Lin, Min},
  journal       = {arXiv preprint arXiv:2503.20783},
  year          = {2025},
  eprint        = {2503.20783},
  archivePrefix = {arXiv},
  primaryClass  = {cs.CL},
  url           = {https://arxiv.org/abs/2503.20783}
}

@inproceedings{RAFAILOV2023,
  title     = {Direct Preference Optimization: Your Language Model is Secretly a Reward Model},
  author    = {Rafailov, Rafael and Sharma, Archit and Mitchell, Eric and Ermon, Stefano and Manning, Christopher D. and Finn, Chelsea},
  booktitle = {Advances in Neural Information Processing Systems},
  volume    = {36},
  year      = {2023},
  url       = {https://proceedings.neurips.cc/paper_files/paper/2023/hash/a85b405ed65c6477a4fe8302b5e06ce7-Abstract-Conference.html}
}

@misc{LI2024,
      title={ReMax: A Simple, Effective, and Efficient Reinforcement Learning Method for Aligning Large Language Models}, 
      author={Ziniu Li and Tian Xu and Yushun Zhang and Zhihang Lin and Yang Yu and Ruoyu Sun and Zhi-Quan Luo},
      year={2024},
      eprint={2310.10505},
      archivePrefix={arXiv},
      primaryClass={cs.LG},
      url={https://arxiv.org/abs/2310.10505}, 
}

@inproceedings{AHMADIAN2024,
  title     = {Back to Basics: Revisiting {REINFORCE}-Style Optimization for Learning from Human Feedback in {LLM}s},
  author    = {Ahmadian, Arash and Cremer, Chris and Gall{\'e}, Matthias and Fadaee, Marzieh and Kreutzer, Julia and Pietquin, Olivier and {\"U}st{\"u}n, Ahmet and Hooker, Sara},
  booktitle = {Proceedings of the 62nd Annual Meeting of the Association for Computational Linguistics (Volume 1: Long Papers)},
  pages     = {12248--12267},
  year      = {2024},
  doi       = {10.18653/v1/2024.acl-long.662},
  url       = {https://aclanthology.org/2024.acl-long.662/}
}

@inproceedings{GAO2023,
  title     = {Scaling Laws for Reward Model Overoptimization},
  author    = {Gao, Leo and Schulman, John and Hilton, Jacob},
  booktitle = {Proceedings of the 40th International Conference on Machine Learning},
  series    = {Proceedings of Machine Learning Research},
  volume    = {202},
  pages     = {10835--10866},
  year      = {2023},
  publisher = {PMLR},
  url       = {https://proceedings.mlr.press/v202/gao23h.html}
}

@article{TYAGI2024,
  title         = {Making Sigmoid-{MSE} Great Again: Output Reset Challenges Softmax Cross-Entropy in Neural Network Classification},
  author        = {Tyagi, Kanishka and Rane, Chinmay and Vaidya, Ketaki and Challgundla, Jeshwanth and Auddy, Soumitro Swapan and Manry, Michael},
  journal       = {arXiv preprint arXiv:2411.11213},
  year          = {2024},
  eprint        = {2411.11213},
  archivePrefix = {arXiv},
  primaryClass  = {cs.LG},
  url           = {https://arxiv.org/abs/2411.11213}
}

@inproceedings{HU2022,
  title     = {{LoRA}: Low-Rank Adaptation of Large Language Models},
  author    = {Hu, Edward J. and Shen, Yelong and Wallis, Phillip and Allen-Zhu, Zeyuan and Li, Yuanzhi and Wang, Shean and Wang, Lu and Chen, Weizhu},
  booktitle = {International Conference on Learning Representations},
  year      = {2022},
  url       = {https://openreview.net/forum?id=nZeVKeeFYf9}
}

@inproceedings{DETTMERS2023,
  title     = {{QLoRA}: Efficient Finetuning of Quantized {LLM}s},
  author    = {Dettmers, Tim and Pagnoni, Artidoro and Holtzman, Ari and Zettlemoyer, Luke},
  booktitle = {Advances in Neural Information Processing Systems},
  volume    = {36},
  year      = {2023},
  url       = {https://proceedings.neurips.cc/paper_files/paper/2023/hash/1feb87871436031bdc0f2beaa62a049b-Abstract.html}
}

@misc{MANGRULKAR2022,
  title        = {{PEFT}: State-of-the-Art Parameter-Efficient Fine-Tuning Methods},
  author       = {Mangrulkar, Sourab and Gugger, Sylvain and Debut, Lysandre and Belkada, Younes and Paul, Sayak},
  year         = {2022},
  howpublished = {Hugging Face software library},
  url          = {https://github.com/huggingface/peft}
}

@inproceedings{HOLTZMAN2020,
  title     = {The Curious Case of Neural Text Degeneration},
  author    = {Holtzman, Ari and Buys, Jan and Du, Li and Forbes, Maxwell and Choi, Yejin},
  booktitle = {International Conference on Learning Representations},
  year      = {2020},
  url       = {https://openreview.net/forum?id=rygGQyrFvH}
}

% ============================================================

\appendix

\section{Experimental Scope Relative to a Larger-Scale RLHF Pipeline}
\label{app:pipeline}

Table~\ref{tab:pipeline} places the present implementation in the
context of a more fully separated, larger-scale RLHF evaluation
pipeline. The comparison clarifies the scope of the evidence rather
than defining a single ``ideal'' implementation.

\begingroup
\small
\setlength{\tabcolsep}{3pt}
\renewcommand{\arraystretch}{1.10}
\setlength{\LTleft}{0pt}
\setlength{\LTright}{0pt}
\begin{longtable}{@{}
  >{\raggedright\arraybackslash}p{0.13\textwidth}
  >{\raggedright\arraybackslash}p{0.26\textwidth}
  >{\raggedright\arraybackslash}p{0.27\textwidth}
  >{\raggedright\arraybackslash}p{0.27\textwidth}@{}}
\caption{Comparison between common larger-scale RLHF practice and the
controlled implementation used in this study.}
\label{tab:pipeline}\\
\toprule
\textbf{Stage} &
\textbf{Common larger-scale practice} &
\textbf{Current study} &
\textbf{Implication} \\
\midrule
\endfirsthead

\multicolumn{4}{l}{\small\itshape Table~\thetable\ continued}\\
\toprule
\textbf{Stage} &
\textbf{Common larger-scale practice} &
\textbf{Current study} &
\textbf{Implication} \\
\midrule
\endhead

\midrule
\multicolumn{4}{r}{\small\itshape Continued on next page}\\
\endfoot

\bottomrule
\endlastfoot

Supervised fine-tuning &
Task-specific SFT checkpoint trained on a substantial instruction set. &
\texttt{Llama-3.2-1B-Instruct} used directly, without an additional
study-specific SFT stage. &
The study does not test whether task-specific SFT changes the relative
behavior of the policy objectives. \\
\addlinespace[2pt]
Preference-data separation &
Disjoint reward-model train/validation/test splits, with separate RL and
final-evaluation prompts. &
Reward model trained on 8,000 \texttt{hh-rlhf} pairs; RL prompts drawn
from the same source dataset. Prompt-level overlap is not documented. &
Possible overlap may inflate training-time reward-model scores and limits
generalization claims; shared conditions do not guarantee that bias
cancels across methods. \\
\addlinespace[2pt]
RL prompt coverage &
Large, diverse RL prompt set with an independent held-out evaluation set. &
Human turns from \texttt{hh-rlhf} used for 500 optimization steps; no
independent held-out human-preference evaluation. &
The results characterize this training setup and prompt distribution,
not broad post-training performance. \\
\addlinespace[2pt]
Policy adaptation &
Full-model tuning or a parameter-efficient configuration selected for
the target model and compute budget. &
LoRA with rank $r=16$ and approximately four million trainable policy
parameters. &
The restricted trainable subspace enables single-GPU experiments but may
change the dynamics relative to full-model or higher-rank tuning. \\
\addlinespace[2pt]
Reward model &
Reward model selected on held-out preference data and evaluated before
policy optimization. &
MLP scalar head on a frozen 4-bit Llama backbone; recorded pairwise
accuracy approximately 60--61\%. &
Scores may be noisy and are not calibrated measures of human utility;
checkpoint selection and evaluation split must match the finalized record. \\
\addlinespace[2pt]
Group-relative configuration &
Group size and response length evaluated through controlled ablations. &
Group size $G=2$, 64-token response cap, and reward-model CPU offload
during generation and policy updates. &
The small group and shorter response cap constrain the group-relative
findings and confound direct cross-family comparisons. \\
\addlinespace[2pt]
Final evaluation &
Fixed checkpoints evaluated on held-out prompts using independent
preference judgments and, where appropriate, multiple evaluators. &
Final logged training-time score from the same frozen reward model,
reported over three seeds. &
The evidence supports matched within-family comparisons, not calibrated
utility, statistical significance, or held-out human-preference claims. \\
\end{longtable}
\endgroup

\section{Illustrative Preference-Pair Format}
\label{app:preference_example}

The following example illustrates the chosen--rejected response format
used to train the Bradley--Terry reward model. It is included only to
clarify the data structure and is not used as evidence about reward-model
quality.

{\fontfamily{pcr}\selectfont
\begin{Verbatim}
chosen:
  Human: I want to learn how to cook.
         Can you give me some basic tips?
  Assistant: Sure! Start with simple recipes.
  Master knife skills early. Control your heat.
  Taste as you go and season at each stage.

rejected:
  Human: I want to learn how to cook.
         Can you give me some basic tips?
  Assistant: Cooking is a great skill.
  Try different recipes and see what you like.
  Practice makes perfect so just keep trying.
\end{Verbatim}
}

\section{Reference Implementation}
\label{app:code}

The listing below makes three implementation choices explicit. First,
zero-advantage tokens follow the detached current log-ratio and therefore
contribute zero direct OR residual. Second, the diagnostic returned by
\texttt{output\_reset} is the squared target magnitude, not the squared
prediction residual. Third, group-relative advantages use a standard
deviation lower-clamped at $10^{-4}$, and zero-variance groups are
excluded before policy-loss aggregation. The standard-deviation
correction is written explicitly so that the implementation does not
depend on a library default.

\begin{Verbatim}[fontsize=\small]
def output_reset(log_ratio, advantages, alpha=0.2):
    """
    Construct detached OR targets.

    advantage > 0: target +alpha until favorable overshoot
    advantage < 0: target -alpha until favorable overshoot
    advantage = 0: target current log-ratio, giving zero residual
    """
    y = log_ratio.detach()

    positive = advantages > 0
    negative = advantages < 0

    target = y.clone()
    target = torch.where(
        positive, torch.full_like(y, +alpha), target
    )
    target = torch.where(
        negative, torch.full_like(y, -alpha), target
    )

    # Favorable overshoot follows the detached current output.
    target = torch.where(positive & (y > +alpha), y, target)
    target = torch.where(negative & (y < -alpha), y, target)

    # Diagnostic only: squared target magnitude, not OR residual.
    target_magnitude_sq = target.square().mean(dim=0)
    return target, target_magnitude_sq


def group_relative_advantages(rm_grouped, eps=1e-4):
    """
    rm_grouped has shape [num_prompts, group_size].
    Zero-variance groups are excluded from returned advantages.
    """
    group_mean = rm_grouped.mean(dim=1, keepdim=True)

    # Explicit Bessel correction; equivalent to unbiased=True.
    raw_group_std = rm_grouped.std(
        dim=1, keepdim=True, correction=1
    )
    valid_group = raw_group_std.squeeze(1) > 0

    denom = raw_group_std.clamp(min=eps)
    advantages = (rm_grouped - group_mean) / denom

    return advantages[valid_group], valid_group
\end{Verbatim}

\end{document}